\documentclass{article}

\usepackage{microtype}
\usepackage{graphicx}
\usepackage{subcaption}
\usepackage{booktabs} 

\usepackage{hyperref}

\usepackage[preprint]{icml2026}

\usepackage{amsmath}
\usepackage{amssymb}
\usepackage{mathtools}
\usepackage{amsthm}

\usepackage[capitalize,noabbrev]{cleveref}

\theoremstyle{plain}
\newtheorem{theorem}{Theorem}[section]
\newtheorem{proposition}[theorem]{Proposition}
\newtheorem{lemma}[theorem]{Lemma}
\newtheorem{corollary}[theorem]{Corollary}
\theoremstyle{definition}

\theoremstyle{remark}
\newtheorem{remark}[theorem]{Remark}

\usepackage[textsize=tiny]{todonotes}

\icmltitlerunning{CCAR: Intrinsic Robustness as an Emergent Geometric Property}

\begin{document}

\twocolumn[
  \icmltitle{Class-Conditional Activation Regularization (CCAR): \\
Intrinsic Robustness as an Emergent Geometric Property}


  \icmlsetsymbol{equal}{*}

\begin{icmlauthorlist}
\icmlauthor{Akash Samanta}{tiu}
\icmlauthor{Manish Pratap Singh}{drdo}
\icmlauthor{Debasis Chaudhuri}{tiu}
\end{icmlauthorlist}

\icmlaffiliation{tiu}{Techno India University, Salt Lake, Kolkata, India}
\icmlaffiliation{drdo}{DRDO Young Scientist Laboratory - CT, Chennai, India}

\icmlcorrespondingauthor{Akash Samanta}{akashsamanta.web@gmail.com}

  \icmlkeywords{Machine Learning, Mechanistic Interpretability, Deep Learning, Feature Geometry, Representation Learning, Learning Representation, Feature Learning, Geometric Regularization, Robustness Convolutional Neural Network}

  \vskip 0.3in
]



\printAffiliationsAndNotice{}  

\begin{abstract}
  Standard supervised learning optimizes for predictive accuracy but remains agnostic to the internal geometry of learned features, often yielding representations that are entangled and brittle. We propose Class-Conditional Activation Regularization (CCAR) to explicitly engineer the feature space, imposing a block-diagonal structure via a soft inductive bias. By shaping the latent representation to confine class energy to orthogonal subspaces, we create an intrinsic geometric scaffold that naturally filters noise and adversarial perturbations. We provide theoretical analysis linking this structural constraint to the maximization of the Fisher Discriminant Ratio, establishing a formal connection between geometric disentanglement and algorithmic stability. Empirically, this approach demonstrates that robustness is an emergent property of a well-engineered feature space, significantly outperforming baselines on label noise and input corruption benchmarks.
\end{abstract}

\section{Introduction}
\label{sec:intro}

The remarkable success of deep neural networks is often attributed to their ability to automatically discover hierarchical representations that disentangle complex data manifolds \cite{bengio2013representation}. Yet, standard supervised training predominantly focuses on the final output: optimizing a prediction error. While this approach effectively minimizes empirical risk, it remains largely agnostic to the internal structure of the learned features. As a result, networks often achieve high accuracy by finding degenerate geometries that are sufficient for classification on clean data but brittle under stress, often collapsing when exposed to label noise \cite{zhang2017understanding}, adversarial perturbations, or distribution shifts \cite{geirhos2020shortcut}.

This disconnect raises a fundamental question: can we explicitly shape the internal geometry of deep representations to be intrinsically robust, without relying on external augmentations or complex training pipelines?

Current approaches to robustness typically intervene at the periphery, such as modifying the data with augmentations \cite{cubuk2019autoaugment} and adversarial training \cite{madry2018towards}, or altering the architecture with specialized modules. While effective, these methods often treat the symptoms of fragile geometry rather than the cause. We hypothesize that robustness is an emergent property of well-structured feature spaces \cite{ilyas2019adversarial}. Specifically, if latent representations are encouraged to organize into distinct, stable subspaces aligned with semantic identity, the network should naturally resist perturbation without needing to encounter every possible noise pattern during training.

In this work, we propose Class-Conditional Geometric Regularization (CCAR), a method to bridge the gap between predictive optimization and structural robustness. Unlike standard regularization that constrains model complexity, CCAR acts as a soft inductive bias on the activations themselves. By designating soft, contiguous subspaces for each class within the high-dimensional feature manifold, we encourage the network to concentrate feature energy into class-specific directions. Crucially, while this enforces class-specific partitions, it does not enforce rigid sparsity; neurons remain polysemantic within their assigned subspaces, and the aggregate geometry becomes significantly more coherent.

\begin{figure}[t]
\centering
\includegraphics[width=\linewidth]{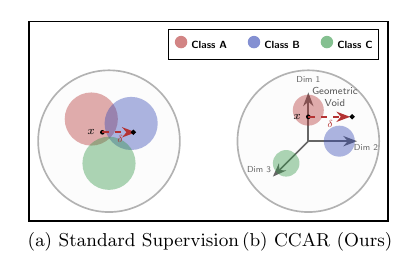}
\vspace{-20pt} 
\caption{Geometric intuition of emergent robustness.
(Left) Standard Supervision: Training typically induces dense, entangled class manifolds with shared boundaries. A small perturbation ($\delta$) easily pushes a sample $x$ (Class A) into a competing manifold (Class B). 
(Right) CCAR (Ours): By enforcing a block-diagonal structure, classes are confined to orthogonal subspaces. This creates a large geometric void (inactive feature space) between manifolds. The same perturbation $\delta$ now projects into this void rather than a competing class, preserving the correct prediction.}
\label{fig:geometric_intuition}
\vspace{-10pt} 
\end{figure}

Our approach draws inspiration from the principle that semantic separation should exist not just at the decision boundary, but throughout the latent representation. We demonstrate that a simple geometric constraint, implemented as a smooth quadratic penalty accompanying the primary supervision, is sufficient to induce this structure. This complementary regularization forces the network to learn features that are not only discriminative but also compact and centered, reducing the manifold instability where adversarial examples and noise typically hide.

We summarize our primary contributions as follows:
\begin{itemize}
    \item \textbf{Geometric Inductive Bias:} We introduce Class-Conditional Activation Regularization (CCAR), a lightweight framework that shapes intermediate feature geometry. We show that this simple intervention significantly improves internal cluster separation and reduces feature overlap.
    \item \textbf{Robustness via Structure:} We provide empirical evidence that shaping geometry leads to emergent robustness. CCAR outperforms standard baselines on label noise, input corruptions, and adversarial attacks, validating that robustness can be achieved by organizing internal representations rather than just optimizing decision boundaries.
    \item \textbf{Theoretical \& Empirical Analysis:} We analyze the mechanism of CCAR through both a theoretical lens by characterizing it as an energy minimization problem under class constraints, and extensive geometric analysis showing that improvements in metrics like the Fisher discriminant ratio directly correlate with downstream robustness.
\end{itemize}

\section{Background and Related Work}
\label{sec:background}

Robust representation learning formally aims to identify an encoder $f_\theta: \mathcal{X} \to \mathbb{R}^D$ such that the induced feature geometry remains stable under perturbations of the input distribution $P(X)$ or the label distribution $P(Y)$. While standard supervised learning optimizes for conditional predictive accuracy, it does not inherently constrain the topology of the latent space $\mathbb{R}^D$, often resulting in representations that are separable but geometrically fragile. This section reviews the geometric properties induced by standard loss functions, the invariance limitations of pairwise contrastive objectives, and existing mechanisms for enforcing robustness.

\subsection{Feature Geometry in Supervised Learning}

In the standard supervised classification setting, a network learns a mapping $f_\theta(x)$ and a linear classifier $W = [w_1, \dots, w_C]^\top$ to minimize the Cross-Entropy (CE) loss:
\begin{equation}
\mathcal{L}_{\text{CE}} = - \frac{1}{N} \sum_{i=1}^N \log \frac{\exp(w_{y_i}^\top f_\theta(x_i))}{\sum_{j=1}^C \exp(w_j^\top f_\theta(x_i))}
\end{equation}
Analytically, minimizing $\mathcal{L}_{\text{CE}}$ drives the angle between the feature vector $f_\theta(x_i)$ and the target weight vector $w_{y_i}$ to zero, while maximizing the angle with $w_{j \neq y_i}$. However, this optimization is under-determined with respect to the feature magnitude and the distribution of energy across dimensions. As long as the dot product is maximized, the encoder is free to use arbitrary subspaces or highly correlated neurons, leading to the feature collapse phenomenon where intra-class variability is compressed only along directions relevant to the specific training set, leaving the representation vulnerable to out-of-distribution noise or adversarial shifts \cite{papyan2020prevalence}.

Auxiliary objectives such as Center Loss \cite{wen2016discriminative} and ArcFace \cite{deng2019arcface} attempt to regularize this space by enforcing metric constraints---minimizing Euclidean distances to centroids or enforcing angular margins on a hypersphere. While effective for recognition tasks, these methods operate on the relative metric between points, rather than constraining the absolute activation patterns of the features themselves.

\subsection{Contrastive Learning and Geometric Invariance}

Contrastive Learning (CL) creates structure by aligning positive pairs $(x, x^+)$ and repelling negatives $x^-$. A canonical objective is the InfoNCE loss:
\begin{equation}
\resizebox{.91\linewidth}{!}{$
    \mathcal{L}_{\text{InfoNCE}} = - \mathbb{E} \left[ \log \frac{\exp(f(x)^\top f(x^+) / \tau)}{\exp(f(x)^\top f(x^+) / \tau) + \sum_{k} \exp(f(x)^\top f(x_k^-) / \tau)} \right]
$}
\end{equation}
A critical limitation of CL objectives for structured feature learning is their invariance to orthogonal transformations. Since the loss depends solely on dot products (or Euclidean distances), the optimal solution is defined only up to a rotation.

\begin{proposition}[Rotation Symmetry]
For any loss $\mathcal{L}$ depending only on pairwise inner products $\langle f(x), f(x') \rangle$, if $f^*$ is a minimizer, then $R f^*$ is also a minimizer for any orthogonal matrix $R \in O(D)$.
\end{proposition}

This symmetry implies that CL cannot enforce axis-aligned disentanglement or specific subspace allocations without additional constraints. Recent efforts such as Non-negative Contrastive Learning (NCL) \cite{wang2024non} attempt to break this symmetry by enforcing non-negativity to induce sparse, axis-aligned features. However, NCL operates within the self-supervised regime to improve interpretability. Even Supervised Contrastive Learning (SupCon) \cite{khosla2020supervised}, which utilizes label information to form positive pairs, inherits this rotational invariance. Consequently, while CL improves global separability, it does not guarantee that specific classes occupy stable, distinct coordinate subspaces.

\subsection{Robustness via Inductive Biases}

Approaches to enforce robustness typically intervene either at the data level or the optimization level. Data-level methods, such as Adversarial Training, solve a min-max problem against a perturbation set $\mathcal{S}$:
\begin{equation}
\min_\theta \mathbb{E} \left[ \max_{\delta \in \mathcal{S}} \mathcal{L}(f_\theta(x + \delta), y) \right]
\end{equation}
This explicitly constructs the decision boundary around worst-case noise but is computationally intensive \cite{madry2018towards}. Optimization-level methods, such as DivideMix \cite{li2020dividemix} or robust loss functions \cite{ghosh2017robust}, attempt to filter or down-weight noisy samples.

A third paradigm, geometric regularization, posits that robustness is an emergent property of feature purity. If feature activations for different classes are confined to orthogonal subspaces, the probability of a random perturbation $\epsilon$ crossing a decision boundary decreases, as it requires specific, high-magnitude projection onto the incorrect class's subspace. This geometric perspective motivates the design of regularizers that explicitly penalize feature leakage into forbidden regions of the latent space. While \cite{lezama2018ole} encourage general orthogonality via covariance regularization, we enforce strict, predefined axis-aligned partitions directly on the activations to ensure robustness. 

\section{Class-Conditional Geometric Regularization}
\label{section3}

\begin{algorithm}[tb]
   \caption{CCAR-L2 Supervised Training}
   \label{alg:ccar}
\begin{algorithmic}
   \STATE {\bfseries Require:} Batch $\mathcal{B} = \{(x_i, y_i)\}_{i=1}^{|\mathcal{B}|}$, Encoder $f_\theta$, Classifier $W$, Strength $\lambda$
   
   \STATE {\bfseries Forward Pass:}
   \STATE $H \leftarrow f_\theta(X)$
   \STATE $Z \leftarrow W H$
   \STATE $\mathcal{L}_{\text{CE}} \leftarrow \text{CrossEntropy}(Z, Y)$
   
   \STATE {\bfseries Geometric Regularization:}
   \STATE Initialize $\mathcal{L}_{\text{reg}} \leftarrow 0$
   \FOR{$i = 1$ {\bfseries to} $|\mathcal{B}|$}
      \STATE $m_i \leftarrow \text{Mask}(y_i)$ 
      \STATE $E_i \leftarrow \| m_i \odot H_i \|_2^2$
      \STATE $\mathcal{L}_{\text{reg}} \leftarrow \mathcal{L}_{\text{reg}} + \frac{1}{\|m_i\|_1} E_i$
   \ENDFOR
   
   \STATE {\bfseries Backward Pass:}
   \STATE $\mathcal{L}_{\text{total}} \leftarrow \mathcal{L}_{\text{CE}} + \frac{\lambda}{|\mathcal{B}|} \mathcal{L}_{\text{reg}}$
   \STATE Update $\theta, W$ via SGD using $\nabla \mathcal{L}_{\text{total}}$
\end{algorithmic}
\end{algorithm}

The standard supervised learning paradigm optimizes an encoder $f_\theta: \mathcal{X} \to \mathbb{R}^D$ and a linear classifier $W$ to minimize the cross-entropy loss. While this objective successfully maximizes the separation between class centroids, it remains agnostic to the internal topology of the feature space $\mathbb{R}^D$. Consequently, the learned representations often exhibit high variance in directions orthogonal to the decision boundary; acting as nuisance directions that contribute nothing to classification but serve as avenues for adversarial perturbations \cite{geirhos2020shortcut,ilyas2019adversarial}.

To remedy this, we introduce Class-Conditional Activation Regularization (CCAR), a geometric framework that explicitly imposes a block-diagonal covariance structure on latent representations. This section details the construction of the subspace constraints, the derivation of the minimum-energy objective, and the optimization procedure.

\subsection{Subspace Partitioning}

We formally define a robust latent geometry as one where class-specific information is confined to distinct, orthogonal subspaces. As illustrated in Figure~\ref{fig:geometric_intuition}, this partitioning transforms the entangled manifolds typical of standard supervision (Figure~\ref{fig:geometric_intuition}(a)) into separated structures (Figure~\ref{fig:geometric_intuition}(b)) where each class is pinned to its designated axes. We partition the feature dimension $D$ into $C$ disjoint, contiguous index sets $\mathcal{P}=\{\mathcal{I}_{1},...,\mathcal{I}_{C}\}$ via a fixed assignment (i.e., $\mathcal{I}_c = \{(c-1)K + 1, \dots, cK\}$). To ensure balanced capacity, we enforce $|\mathcal{I}_c| = K = \lfloor D/C \rfloor$; any remainder dimensions $D \pmod C$ are excluded from all active sets and thus globally suppressed..

For an input $x$ associated with label $y$, the ideal feature representation $h = f_\theta(x)$ should reside strictly within the subspace spanned by the basis vectors in $\mathcal{I}_y$. We define the \textit{forbidden region} $\mathcal{O}_y$ as the orthogonal complement of the active set:
\begin{equation}
\mathcal{O}_y = \{1, \dots, D\} \setminus \mathcal{I}_y = \bigcup_{k \neq y} \mathcal{I}_k
\end{equation}
To operationalize this constraint, we define a class-conditional projection operator $M(y) \in \{0, 1\}^D$ that isolates the forbidden dimensions:
\begin{equation}
(M(y))_j = \mathbb{I}(j \in \mathcal{O}_y)
\end{equation}

\begin{figure*}[!t]
\vskip 0.2in
\begin{center}
    \begin{minipage}[b]{0.32\textwidth}
        \centering
        \includegraphics[width=\linewidth]{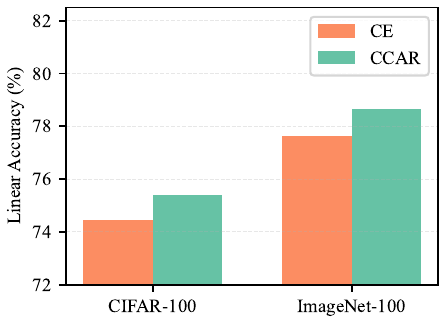}
        \centerline{\small (a) Linear Accuracy}
    \end{minipage}
    \hfill
    \begin{minipage}[b]{0.32\textwidth}
        \centering
        \includegraphics[width=\linewidth]{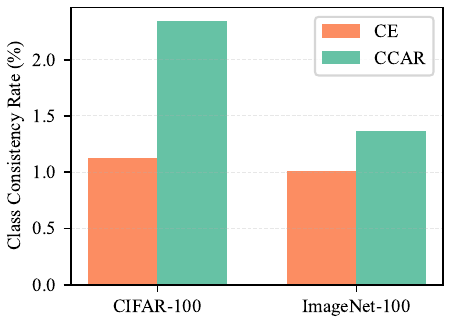}
        \centerline{\small (b) Consistency Rate}
    \end{minipage}
    \hfill
    \begin{minipage}[b]{0.32\textwidth}
        \centering
        \includegraphics[width=\linewidth]{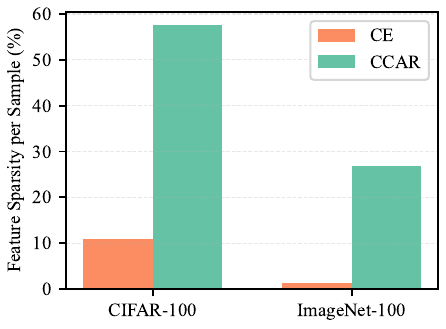}
        \centerline{\small (c) Feature Sparsity}
    \end{minipage}

    \caption{Comparison between Cross-Entropy (CE) and Class-Conditional Activation Regularization (CCAR): 
    (a) linear probing accuracy; 
    (b) class consistency rate, measuring the proportion of activated samples that belong to their most frequent class along each feature dimension; 
    (c) feature sparsity, the average proportion of zero elements ($|x|<1\text{e-}5$) in the features of each test sample. CCAR significantly improves consistency and sparsity without sacrificing accuracy.}
    \label{fig:geometric_metrics}
\end{center}
\vskip -0.2in
\end{figure*}

\textbf{Remark on Polysemanticity.} It is important to distinguish between \textit{subspace constraints} and \textit{neuron monosemanticity}. CCAR allocates a dedicated subspace $\mathcal{I}_y$ to each class, but it imposes no sparsity constraints \textit{within} that subspace. The network retains the freedom to learn distributed, polysemantic representations, where single neurons encode multiple intra-class attributes \cite{olah2020zoom, elhage2022toy}, within the active block. This structure preserves the coding efficiency of high-dimensional spaces while eliminating cross-class interference.

\subsection{The CCAR-L2 Objective}
\label{l2choice}

We formulate the regularization as an energy minimization problem on the forbidden projection. Let $h = f_\theta(x)$ be the feature representation. The geometric loss is defined as the mean squared energy of the residual energy into forbidden dimensions:
\begin{equation}
\mathcal{L}_{\text{CCAR}}(x, y) = \frac{1}{|\mathcal{O}_y|} \| M(y) \odot h \|_2^2 = \frac{1}{|\mathcal{O}_y|} \sum_{j \in \mathcal{O}_y} h_j^2
\end{equation}
We adopt the squared $L_2$ norm rather than sparsity-inducing norms (e.g., $L_1$) to ensure optimization stability. The $L_2$ penalty provides a Lipschitz continuous gradient field, offering a restorative force that scales linearly with deviation. This prevents the oscillatory convergence near the manifold origin that is characteristic of non-smooth penalties \cite{bertsekas1999nonlinear}, facilitating precise convergence to the subspace. A comprehensive theoretical and empirical comparison justifying the choice of $L_2$ over alternative norms is provided in Appendix \ref{app:regularization}.

\subsection{Optimization Algorithm}

The final training objective is a composite loss that balances discriminative performance with geometric compression. We combine the standard Cross-Entropy loss ($\mathcal{L}_{\text{CE}}$) with the CCAR regularizer:
\begin{equation}
\min_{\theta, W} \mathcal{J} = \mathbb{E}_{(x, y) \sim \mathcal{D}} \left[ \mathcal{L}_{\text{CE}}(W f_\theta(x), y) + \lambda \mathcal{L}_{\text{CCAR}}(f_\theta(x), y) \right]
\end{equation}
where $\lambda$ is a scalar hyperparameter controlling the regularization strength.
To ensure the practical feasibility of this objective, we characterize the optimization landscape under the $L_2$ penalty. Unlike non-smooth regularizers, the quadratic nature of our constraint ensures stable convergence behavior.

\textbf{Computational Complexity.} The overhead introduced by CCAR is negligible. The mask generation and element-wise multiplication are $O(D)$ operations, which can be fully parallelized on GPU hardware. Unlike contrastive methods that require $O(N^2)$ pairwise computations or costly negative mining \cite{chen2020simple}, CCAR operates independently on each sample, scaling linearly $O(N)$ with batch size.

\section{Theoretical Analysis}
\label{section4}
In this section, we analyze the geometric properties induced by the CCAR-L2 objective. We demonstrate that minimizing the forbidden energy is statistically equivalent to maximizing the Fisher Discriminant Ratio in the latent space. Furthermore, we derive a geometric stability bound, proving that the block-diagonal structure creates a certified margin against additive perturbations.

\subsection{Covariance Compression and Fisher Ratio}

We first characterize the effect of CCAR on the statistical structure of the learned features. Let $\Sigma_c = \mathbb{E}[ (h - \mu_c)(h - \mu_c)^\top | y=c ]$ denote the class-conditional covariance matrix for class $c$. In unregularized networks, $\Sigma_c$ is typically full-rank, indicating variance in all directions.

\begin{theorem}[Trace Minimization]
\label{thm:trace_minimization}
Assuming class means are aligned with their active subspaces ($P_{\mathcal{O}_c}\mu_c = 0$), minimizing the expected CCAR-L2 loss is equivalent to minimizing the trace of the class-conditional covariance matrix projected onto the forbidden subspace $\mathcal{O}_c$.
\end{theorem} Proof of this property is provided in Appendix \ref{proof:theorem1}. \vspace{5pt} 

\begin{figure}[t]
\vskip 0.2in
\begin{center}
    \centerline{\includegraphics[width=\columnwidth]{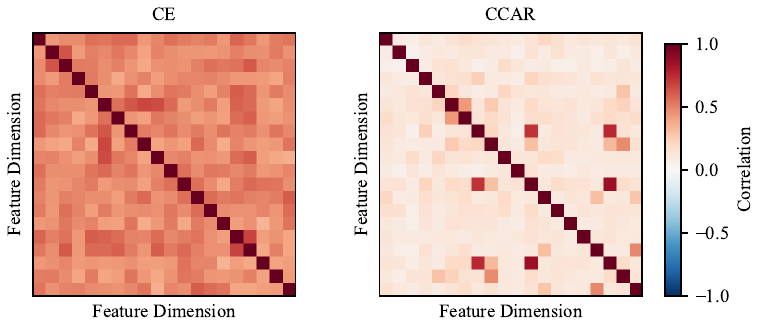}}
    \caption{Dimensional correlation matrix of 20 random features on ImageNet-100. The heatmap visualizes the pairwise cosine similarity between feature activation vectors across the test set. The Cross-Entropy baseline (left) exhibits diffuse off-diagonal correlations, indicating significant entanglement between feature dimensions. In contrast, CCAR (right) effectively suppresses this interference, yielding a sharp block-diagonal topology that maximizes feature orthogonality.}
    \label{fig:correlation_matrix}
\end{center}
\vskip -0.2in
\end{figure}

\begin{remark}
This theorem establishes a direct link to Fisher Discriminant Analysis (FDA) \cite{fisher1936use}. The Fisher Ratio is defined as $J = \frac{\text{Tr}(\Sigma_b)}{\text{Tr}(\Sigma_w)}$, where $\Sigma_b$ is the inter-class scatter and $\Sigma_w$ is the intra-class scatter. By strictly penalizing $\sum_{j \in \mathcal{O}_c} h_j^2$, CCAR forces the diagonal elements of $\Sigma_c$ (which constitute $\Sigma_w$) to zero for all forbidden indices. Crucially, it does not constrain the variance within the active subspace $\mathcal{I}_c$. This selective compression \textbf{minimizes} the denominator of the Fisher Ratio without collapsing the useful intra-class diversity required for fine-grained discrimination. Unlike standard FDA, which seeks a linear projection after feature extraction, CCAR embeds this optimality directly into the encoder's topology.
\end{remark}

\subsection{Perturbation Stability}

We now turn to the robustness implications of this geometry. A key failure mode of deep networks is sensitivity to small, adversarial perturbations that cross decision boundaries \cite{szegedy2013intriguing,goodfellow2014explaining}. We show that CCAR creates a geometric margin against such shifts.

To analyze the stability of the latent manifold, we consider the model in its converged state. We assume the linear classifier $W$ has reached a state of alignment with the partitioned subspaces $\mathcal{P}$, such that each weight vector $w_k$ resides in its designated subspace $\mathcal{I}_k$.

\begin{theorem}[Perturbation Stability Bound]
\label{thm:stability_bound}
Ideally, let a clean feature $h$ lie strictly within the active subspace $\mathcal{I}_y$. For an additive perturbation $\delta$ to induce a misclassification into a target class $k \ne y$, the magnitude of the perturbation projected onto the forbidden subspace $\mathcal{I}_k$ must exceed the effective geometric margin defined by the ratio of the correct class signal to the target class weight sensitivity.
\end{theorem} See Appendix \ref{proof:theorem2} for the proof. \vspace{5pt}

\begin{figure}[t]
\begin{center}
\includegraphics[width=0.35\textwidth]{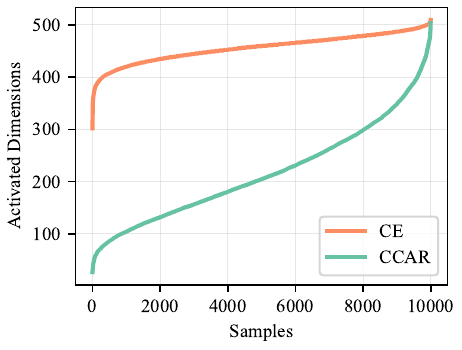}
\caption{The number of activated dimensions of the representations learned by CE and CCAR on CIFAR-100. CCAR (Green) significantly reduces the effective dimensionality, indicating that the model relies on a compact, sparse set of features for prediction compared to the dense activations of the Cross-Entropy baseline.}
\label{fig:activated_dims}
\end{center}
\end{figure}

\begin{corollary}[Certified Robustness]
\label{corr:robustness}
Let the minimum geometric margin be defined as $\tau = \min_{k \neq y} \frac{w_y^\top h}{\|w_y - w_k\|_2}$. If the perturbation norm satisfies $\|\delta\|_2 < \tau$ and clean features are confined to $\mathcal{I}_y$, then the classifier is provably robust to misclassification.
\end{corollary}
\begin{proof}
Misclassification into class $k$ occurs if $(w_k - w_y)^\top (h + \delta) > 0$. Since $w_k \perp h$ and $h \in \mathcal{I}_y$, this simplifies to $(w_k - w_y)^\top \delta > w_y^\top h$. By Cauchy-Schwarz, we have $\|\delta\|_2 \|w_k - w_y\|_2 \ge (w_k - w_y)^\top \delta$. Thus, a successful attack requires $\|\delta\|_2 > \frac{w_y^\top h}{\|w_k - w_y\|_2} \ge \tau$. Conversely, $\|\delta\|_2 < \tau$ guarantees stability.
\end{proof}

While the preceding results establish bounds for arbitrary shifts, we further characterize the resilience of the learned geometry to isotropic (directionless) noise. 

\begin{table*}[!t]
\caption{Comparison of Linear Probing Accuracy (\%) on CIFAR-100 and ImageNet-100 under symmetric label noise. CCAR demonstrates superior robustness across all noise levels, maintaining high performance even in extreme noise regimes where baselines degrade. (Best results in bold).}
\label{tab:main_results_noise}
\begin{center}
\begin{small}
\begin{tabular}{cccccc}
\toprule
\multicolumn{6}{c}{\textbf{\textsc{CIFAR-100}}} \\
\midrule
\textsc{Noise (\%)} & \textsc{Cross-Entropy} & \textsc{Center Loss} & \textsc{Contrastive (CL)} & \textsc{NCL} & \textbf{\textsc{CCAR (Ours)}} \\
\midrule
0  & 74.46 $\pm$ 0.24 & 75.39 $\pm$ 0.15 & 75.65 $\pm$ 0.31 & 74.21 $\pm$ 0.12 & \textbf{75.69 $\pm$ 0.18} \\
10 & 72.90 $\pm$ 0.11 & 74.33 $\pm$ 0.29 & 75.26 $\pm$ 0.22 & 73.68 $\pm$ 0.09 & \textbf{75.41 $\pm$ 0.25} \\
20 & 72.10 $\pm$ 0.35 & 74.19 $\pm$ 0.13 & 75.39 $\pm$ 0.16 & 73.28 $\pm$ 0.28 & \textbf{75.26 $\pm$ 0.07} \\
30 & 71.52 $\pm$ 0.18 & 73.75 $\pm$ 0.32 & 75.23 $\pm$ 0.29 & 73.10 $\pm$ 0.15 & \textbf{75.26 $\pm$ 0.33} \\
40 & 70.97 $\pm$ 0.27 & 74.15 $\pm$ 0.08 & 74.94 $\pm$ 0.24 & 73.17 $\pm$ 0.36 & \textbf{75.25 $\pm$ 0.12} \\
50 & 70.11 $\pm$ 0.13 & 73.58 $\pm$ 0.25 & 74.84 $\pm$ 0.10 & 72.62 $\pm$ 0.21 & \textbf{75.15 $\pm$ 0.29} \\
60 & 68.59 $\pm$ 0.42 & 73.23 $\pm$ 0.33 & 74.71 $\pm$ 0.19 & 72.61 $\pm$ 0.14 & \textbf{75.11 $\pm$ 0.22} \\
70 & 66.95 $\pm$ 0.21 & 73.33 $\pm$ 0.17 & 73.94 $\pm$ 0.35 & 72.37 $\pm$ 0.26 & \textbf{75.00 $\pm$ 0.11} \\
80 & 63.37 $\pm$ 0.38 & 72.21 $\pm$ 0.21 & 73.87 $\pm$ 0.13 & 71.26 $\pm$ 0.30 & \textbf{74.67 $\pm$ 0.27} \\
90 & 52.74 $\pm$ 0.25 & 69.70 $\pm$ 0.41 & 72.34 $\pm$ 0.26 & 67.27 $\pm$ 0.19 & \textbf{74.21 $\pm$ 0.34} \\
\midrule
\multicolumn{6}{c}{\textbf{\textsc{ImageNet-100}}} \\
\midrule
0  & 77.64 $\pm$ 0.15 & 78.34 $\pm$ 0.22 & 79.84 $\pm$ 0.09 & 79.14 $\pm$ 0.31 & \textbf{80.62 $\pm$ 0.14} \\
10 & 76.38 $\pm$ 0.29 & 77.66 $\pm$ 0.11 & 79.50 $\pm$ 0.25 & 78.88 $\pm$ 0.17 & \textbf{80.76 $\pm$ 0.21} \\
20 & 76.22 $\pm$ 0.13 & 77.28 $\pm$ 0.34 & 79.52 $\pm$ 0.19 & 78.68 $\pm$ 0.22 & \textbf{80.60 $\pm$ 0.10} \\
30 & 76.08 $\pm$ 0.36 & 77.12 $\pm$ 0.18 & 79.36 $\pm$ 0.27 & 79.02 $\pm$ 0.13 & \textbf{80.54 $\pm$ 0.28} \\
40 & 76.20 $\pm$ 0.22 & 76.86 $\pm$ 0.29 & 79.34 $\pm$ 0.14 & 78.54 $\pm$ 0.35 & \textbf{80.52 $\pm$ 0.08} \\
50 & 75.68 $\pm$ 0.19 & 77.16 $\pm$ 0.12 & 79.56 $\pm$ 0.30 & 78.56 $\pm$ 0.26 & \textbf{80.42 $\pm$ 0.31} \\
60 & 75.42 $\pm$ 0.31 & 76.84 $\pm$ 0.24 & 79.54 $\pm$ 0.16 & 78.64 $\pm$ 0.09 & \textbf{80.42 $\pm$ 0.23} \\
70 & 75.24 $\pm$ 0.14 & 76.92 $\pm$ 0.37 & 79.44 $\pm$ 0.21 & 78.80 $\pm$ 0.18 & \textbf{80.14 $\pm$ 0.12} \\
80 & 73.66 $\pm$ 0.26 & 75.82 $\pm$ 0.15 & 79.28 $\pm$ 0.33 & 78.56 $\pm$ 0.24 & \textbf{80.28 $\pm$ 0.36} \\
90 & 69.08 $\pm$ 0.39 & 74.08 $\pm$ 0.28 & 79.50 $\pm$ 0.12 & 78.26 $\pm$ 0.38 & \textbf{79.98 $\pm$ 0.64} \\
\bottomrule
\end{tabular}
\end{small}
\end{center}
\vskip -0.1in
\end{table*}

\begin{lemma}[Dimensional Noise Filtering]
\label{lem:noise_filtering}
Let $\delta \in \mathbb{R}^D$ be a perturbation vector sampled from an isotropic Gaussian distribution $\mathcal{N}(0, \frac{\sigma^2}{D} I)$. For a class-conditional subspace $\mathcal{I}_k$ of dimension $K = D/C$, the projected noise energy $\|\delta_{\mathcal{I}_k}\|^2$ follows a scaled $\chi^2$ distribution. As $D \to \infty$ with $C$ fixed, provided the geometric margin satisfies $\tau > \frac{\sigma^2}{C}$, the probability that the perturbation bridges the margin vanishes exponentially: 
\begin{equation}
P(\|\delta_{\mathcal{I}_k}\|^2 \ge \tau) \le \exp(-D \cdot \mathcal{A})
\end{equation}
where $\mathcal{A}$ is a strictly positive rate constant depending on $\tau$ and $\sigma^2$.
\end{lemma}
Proof of this concentration property is provided in Appendix \ref{proof:lemma1}.


\textbf{Geometric Interpretation.} In standard training, decision boundaries are often close to the data manifold in arbitrary directions \cite{fawzi2018empirical}. Theorem 4.3 implies that CCAR essentially concentrates the class manifold within $\mathcal{I}_y$, pushing the decision boundaries for all other classes $k$ far away in the orthogonal directions $\mathcal{I}_k$. To flip a prediction, an adversary cannot simply perturb the input slightly; they must inject significant energy specifically into the subspace $\mathcal{I}_k$. As the dimension $D$ increases, the probability of random noise aligning with a specific forbidden subspace $\mathcal{I}_k$ diminishes, providing inherent robustness to isotropic noise \cite{fawzi2016robustness}.

\section{Empirical Analysis}

In this section, leveraging the inductive biases of Class-Conditional Activation Regularization (CCAR), we evaluate the efficacy of our method across three distinct axes: robustness to supervision corruption, geometric mechanism verification, and emergent input stability. We demonstrate that while standard Cross-Entropy (CE) training is agnostic to latent topology, CCAR successfully imposes a stable manifold structure that persists even under extreme noise regimes. Detailed hyperparameter configurations and implementation specifics are provided in Appendix \ref{app:experiment_setup}.

\subsection{Robustness to Supervision Corruption}

Standard supervised learning predominantly optimizes for empirical risk, a process that often leads the encoder to prioritize the memorization of stochastic associations when labels are corrupted \cite{zhang2017understanding}. We posit that CCAR provides a geometric safeguard against such overfitting by confining class-specific information to disjoint orthogonal subspaces. To evaluate this, we conduct experiments on CIFAR-100 and ImageNet-100 under symmetric label noise ranging from 0\% to 90\%.

\textbf{Baselines.} On both CIFAR-100 \cite{krizhevsky2009learning}. and ImageNet-100 \cite{deng2009imagenet}, we benchmark against the standard Cross-Entropy (CE) baseline and extend the evaluation to include strong geometric and contrastive baselines: Center Loss \cite{wen2016discriminative}, which minimizes intra-class variance; SimCLR \cite{chen2020simple}, representing self-supervised contrastive learning; and Non-negative Contrastive Learning (NCL) \cite{wang2024non}, which enforces sparsity. Implementation details for these auxiliary objectives are provided in Appendix \ref{app:baseline_implementation}.

\begin{figure*}[t]
\begin{center}
    \begin{minipage}[b]{0.32\textwidth}
        \centering
        \includegraphics[width=\linewidth]{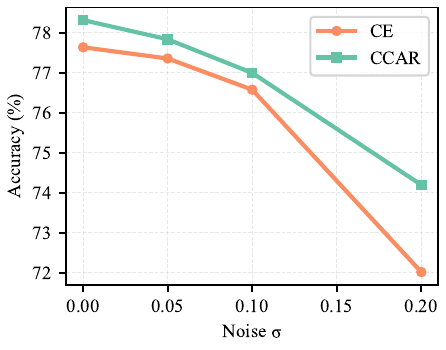}
        \centerline{\small (a) Isotropic Gaussian Noise}
    \end{minipage}
    \hfill
    \begin{minipage}[b]{0.32\textwidth}
        \centering
        \includegraphics[width=\linewidth]{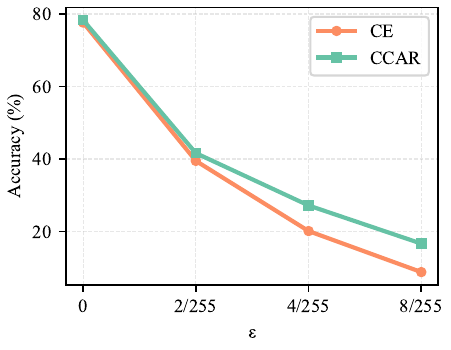}
        \centerline{\small (b) Single-Step FGSM Attack}
    \end{minipage}
    \hfill
    \begin{minipage}[b]{0.32\textwidth}
        \centering
        \includegraphics[width=\linewidth]{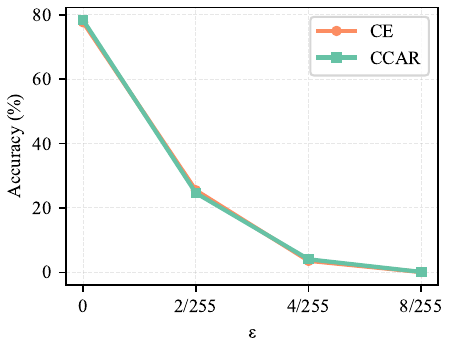}
        \centerline{\small (c) Iterative PGD-20 Attack}
    \end{minipage}
    
    \caption{Robustness analysis on ImageNet-100. We compare the classification accuracy of CCAR against the CE baseline under \textbf{(a)} Isotropic Gaussian Noise, \textbf{(b)} Single-step FGSM perturbations, and \textbf{(c)} Iterative PGD-20 attacks.}
    \label{fig:stability_analysis}
\end{center}
\vskip -0.1in
\end{figure*}

\textbf{Results.} As visualized in Figure \ref{fig:geometric_metrics}(a) and summarized in Table \ref{tab:main_results_noise}, CCAR demonstrates superior stability across critical noise regimes. On clean data (0\% noise), CCAR achieves comparable accuracy to the baselines, indicating that the subspace constraints ($K=D/C$) do not inhibit representational capacity. However, as the noise rate increases, the divergence becomes pronounced.

On CIFAR-100, while Center Loss and Contrastive methods offer marginal improvements over CE, they still degrade significantly at high noise levels. In contrast, CCAR maintains robust discriminative power, outperforming the closest baseline by a significant margin at 80\% noise. This trend holds on ImageNet-100, where the unregularized baseline collapses under label memorization, while CCAR effectively filters the high-energy, incoherent signals required to fit random labels. A comprehensive analysis of the training dynamics and further noise intervals is provided in Appendix \ref{g4}.

\subsection{Geometric Mechanism Verification}

We proceed to verify whether the observed robustness arises from the geometric structural changes predicted by our theory. Specifically, we analyze the latent topology through three complementary lenses: semantic consistency, feature sparsity, and subspace partitioning. Detailed definitions and implementation specifics for these metrics are provided in Appendix \ref{g2}.

\textbf{Disentanglement and Sparsity.} As illustrated in Figure \ref{fig:geometric_metrics}, CCAR induces a fundamental shift in feature organization compared to the Cross-Entropy (CE) baseline. Figure \ref{fig:correlation_matrix} reveals that while CE features exhibit a diffuse correlation structure (indicating significant inter-class entanglement), CCAR features form a sharp block-diagonal topology. This visual disentanglement is quantified in Figure 2(c): CCAR achieves substantially higher feature sparsity per sample (reaching $\approx 58\%$ on CIFAR-100 and $\approx 28\%$ on ImageNet-100), whereas the baseline exhibits significantly lower sparsity ($\approx 11\%$ and $<2\%$, respectively). Crucially, Figure \ref{fig:geometric_metrics}(b) confirms that this sparsity is not destructive; the Class Consistency Rate remains significantly higher for CCAR, implying that the few active neurons are highly semantically relevant to the input class.

\textbf{Emergent Subspace Partitioning.} To confirm that this sparsity corresponds to the specific subspace constraints defined in Theorem~\ref{thm:stability_bound}, we analyze the distribution of activated dimensions across the validation set. This confirms that the encoder has successfully partitioned the latent space $\mathbb{R}^D$ into disjoint active subspaces $\mathcal{I}_y$, effectively creating the geometric margins against perturbation predicted by Theorem~\ref{thm:stability_bound}.

\textbf{Results.} Figure \ref{fig:activated_dims} provides the most direct empirical evidence of the geometric mechanism. The baseline model (Orange) utilizes nearly the entire feature dimension ($D=512$) for every sample, indicative of a diffuse but entangled representation where noise and signal are inextricably mixed. In stark contrast, CCAR (Green) exhibits a distinct, block-sparse sparsity profile, where samples consistently activate only a fraction of the network capacity. This confirms that the encoder has successfully partitioned the latent space $\mathbb{R}^D$ into disjoint active subspaces $\mathcal{I}_y$, effectively creating the geometric margins against perturbation predicted by Theorem \ref{thm:stability_bound}.

\subsection{Emergent Input Stability and Retrieval}
\label{sec:input_stability}

Finally, we evaluate whether the geometric margin established in Theorem~\ref{thm:stability_bound} and the statistical filtering properties of Lemma~\ref{lem:noise_filtering} extend robustness to unseen input perturbations. We subject models to three categories of input corruption: isotropic Gaussian noise ($\sigma \in \{0.05, 0.1, 0.2\}$), single-step adversarial perturbations (FGSM) (Goodfellow et al., 2014) at $\epsilon \in \{2/255, 4/255, 8/255\}$, and iterative adversarial attacks (PGD-20) (Madry et al., 2018) with step size $\alpha = \epsilon/4$.

\begin{figure}[!t]
\begin{center}
    \includegraphics[width=0.7\columnwidth]{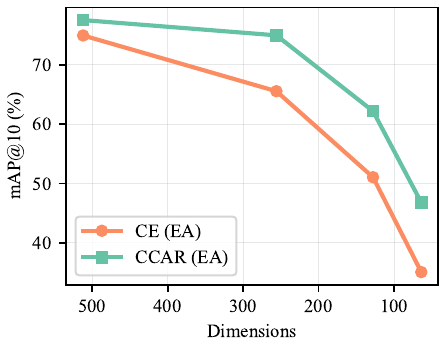}
    \caption{Image retrieval with CE and CCAR features on ImageNet-100. CCAR achieves a significantly higher mAP@10 score, indicating that geometric disentanglement improves semantic consistency alongside robustness.}
    \label{fig:retrieval_map}
\end{center}
\vskip -0.1in
\end{figure}

\textbf{Stability Analysis.} Figure \ref{fig:stability_analysis} illustrates the performance under these corruptions. In the case of stochastic Gaussian noise (Figure \ref{fig:stability_analysis}(a)) and single-step FGSM attacks (Figure \ref{fig:stability_analysis}(b)), CCAR demonstrates a distinct stability advantage, consistently outperforming the Cross-Entropy baseline. This result empirically supports Theorem \ref{thm:stability_bound}: minimizing activation energy in forbidden directions naturally increases the perturbation budget required to push a sample into a competing class subspace. However, under the iterative PGD-20 attack (Figure \ref{fig:stability_analysis}(c)), CCAR performs comparably to the baseline. This suggests that while the induced geometry
provides inherent resilience against single-step and stochastic perturbations, it does not fundamentally alter the local curvature in a manner that resists iterative optimization. 

\textbf{Semantic Retrieval Quality.} A common trade-off in robust learning is that increased stability often comes at the cost of feature discriminability (clean accuracy). We investigate this by evaluating the semantic consistency of the learned representations through image retrieval (mAP@10). As shown in Figure \ref{fig:retrieval_map}, CCAR breaks this trade-off, achieving a significantly higher mAP score compared to the baseline. This confirms that the geometric constraints do not rigidify the representation; rather, by explicitly decoupling class subspaces, CCAR enhances the metric quality of the features, making them both more robust and semantically richer. Qualitative visualizations of class-specific neurons are provided in Appendix \ref{app:visual_interpretability}.

\section{Conclusion}

Despite the predictive success of standard supervised learning, the resulting latent representations often lack the structural organization necessary for intrinsic robustness. Inspired by the principles of geometric disentanglement, we proposed Class-Conditional Activation Regularization (CCAR) to explicitly impose a block-diagonal covariance structure on latent features. With minimal computational overhead and no additional data requirements, CCAR-L2 effectively confines class-specific information to orthogonal subspaces, significantly enhancing robustness against both supervision corruption and unseen input perturbations.

Our work provided a comprehensive theoretical characterization of this mechanism, establishing formal links between forbidden energy minimization and the maximization of the Fisher Discriminant Ratio. Empirically, we demonstrated that this geometric organization enables models to maintain high discriminative power even under extreme label noise regimes where standard Empirical Risk Minimization collapses. Ultimately, we believe that explicitly shaping latent topology serves as a powerful inductive bias, offering a promising alternative to purely data-centric robustness strategies in the pursuit of more reliable deep learning methodologies.

\section*{Impact Statement}

Standard supervised learning prioritizes the final prediction error but often neglects the structural integrity of the learned features. We propose Class-Conditional Activation Regularization to correct this imbalance by imposing a soft inductive bias that actively shapes the internal representation. This geometric foundation ensures that in times of adversity, such as exposure to high label noise or input perturbations, the model maintains a coherent and robust scaffold that prevents collapse \cite{geirhos2020shortcut}. By achieving this stability without the heavy computational cost of adversarial training \cite{madry2018towards}, our approach offers a sustainable path toward reliable AI. Furthermore, this enforced disentanglement aligns latent dimensions with semantic meaning, opening new doors for mechanistic interpretability and transparent auditing of decision-making processes \cite{elhage2022toy}. There are many potential societal consequences of our work, none which we feel must be specifically highlighted here.

\bibliography{example_paper}

@inproceedings{alain2016understanding,
  title={Understanding intermediate layers using linear classifier probes},
  author={Alain, Guillaume and Bengio, Yoshua},
  booktitle={International Conference on Learning Representations},
  year={2016}
}

@inproceedings{bau2017network,
  title={Network dissection: Quantifying interpretability of deep visual representations},
  author={Bau, David and Zhou, Bolei and Khosla, Aditya and Oliva, Aude and Torralba, Antonio},
  booktitle={Proceedings of the IEEE Conference on Computer Vision and Pattern Recognition},
  pages={6541--6549},
  year={2017}
}

@article{bengio2013representation,
  title={Representation learning: A review and new perspectives},
  author={Bengio, Yoshua and Courville, Aaron and Vincent, Pascal},
  journal={IEEE Transactions on Pattern Analysis and Machine Intelligence},
  volume={35},
  number={8},
  pages={1798--1828},
  year={2013},
  publisher={IEEE}
}

@book{bertsekas1999nonlinear,
  title={Nonlinear programming},
  author={Bertsekas, Dimitri P},
  year={1999},
  publisher={Athena scientific}
}

@inproceedings{chen2020simple,
  title={A simple framework for contrastive learning of visual representations},
  author={Chen, Ting and Kornblith, Simon and Norouzi, Mohammad and Hinton, Geoffrey},
  booktitle={International Conference on Machine Learning},
  pages={1597--1607},
  year={2020},
  organization={PMLR}
}

@inproceedings{cubuk2019autoaugment,
  title={AutoAugment: Learning augmentation strategies from data},
  author={Cubuk, Ekin D and Zoph, Barret and Mane, Dandelion and Vasudevan, Vijay and Le, Quoc V},
  booktitle={Proceedings of the IEEE/CVF Conference on Computer Vision and Pattern Recognition},
  pages={113--123},
  year={2019}
}

@inproceedings{deng2019arcface,
  title={ArcFace: Additive angular margin loss for deep face recognition},
  author={Deng, Jiankang and Guo, Jia and Xue, Niannan and Zafeiriou, Stefanos},
  booktitle={Proceedings of the IEEE/CVF Conference on Computer Vision and Pattern Recognition},
  pages={4690--4699},
  year={2019}
}

@inproceedings{deng2009imagenet,
  title={ImageNet: A large-scale hierarchical image database},
  author={Deng, Jia and Dong, Wei and Socher, Richard and Li, Li-Jia and Li, Kai and Fei-Fei, Li},
  booktitle={Proceedings of the IEEE Conference on Computer Vision and Pattern Recognition},
  pages={248--255},
  year={2009},
  organization={IEEE}
}

@article{elhage2022toy,
  title={Toy models of superposition},
  author={Elhage, Nelson and Hume, Tristan and Olsson, Catherine and Schiefer, Nicholas and Henighan, Tom and Kravec, Shauna and others},
  journal={Transformer Circuits Thread},
  year={2022},
  url={https://transformer-circuits.pub/2022/toy_model/index.html}
}

@inproceedings{fawzi2016robustness,
  title={Robustness of classifiers: from adversarial to random noise},
  author={Fawzi, Alhussein and Moosavi-Dezfooli, Seyed-Mohsen and Frossard, Pascal},
  booktitle={Advances in Neural Information Processing Systems},
  volume={29},
  year={2016}
}

@inproceedings{fawzi2018empirical,
  title={Empirical study of the topology and geometry of deep networks},
  author={Fawzi, Alhussein and Moosavi-Dezfooli, Seyed-Mohsen and Frossard, Pascal and Soatto, Stefano},
  booktitle={Proceedings of the IEEE Conference on Computer Vision and Pattern Recognition},
  pages={3762--3770},
  year={2018}
}

@article{fisher1936use,
  title={The use of multiple measurements in taxonomic problems},
  author={Fisher, Ronald A},
  journal={Annals of Eugenics},
  volume={7},
  number={2},
  pages={179--188},
  year={1936},
  publisher={Wiley Online Library}
}

@article{geirhos2020shortcut,
  title={Shortcut learning in deep neural networks},
  author={Geirhos, Robert and Jacobsen, J{\"o}rn-Henrik and Michaelis, Claudio and Zemel, Richard and Brendel, Wieland and Bethge, Matthias and Wichmann, Felix A},
  journal={Nature Machine Intelligence},
  volume={2},
  number={11},
  pages={665--673},
  year={2020},
  publisher={Nature Publishing Group}
}

@inproceedings{ghosh2017robust,
  title={Robust loss functions under label noise for deep neural networks},
  author={Ghosh, Aritra and Kumar, Himanshu and Sastry, P Shanti},
  booktitle={AAAI Conference on Artificial Intelligence},
  volume={31},
  number={1},
  year={2017}
}

@inproceedings{goodfellow2014explaining,
  title={Explaining and harnessing adversarial examples},
  author={Goodfellow, Ian J and Shlens, Jonathon and Szegedy, Christian},
  booktitle={International Conference on Learning Representations},
  year={2014}
}

@inproceedings{he2016deep,
  title={Deep residual learning for image recognition},
  author={He, Kaiming and Zhang, Xiangyu and Ren, Shaoqing and Sun, Jian},
  booktitle={Proceedings of the IEEE Conference on Computer Vision and Pattern Recognition},
  pages={770--778},
  year={2016}
}

@inproceedings{higgins2017beta,
  title={beta-VAE: Learning Basic Visual Concepts with a Constrained Variational Framework},
  author={Higgins, Irina and Matthey, Loic and Pal, Arka and Burgess, Christopher and Glorot, Xavier and Botvinick, Matthew and Mohamed, Shakir and Lerchner, Alexander},
  booktitle={International Conference on Learning Representations},
  year={2017}
}

@inproceedings{ilyas2019adversarial,
  title={Adversarial examples are not bugs, they are features},
  author={Ilyas, Andrew and Santurkar, Shibani and Tsipras, Dimitris and Engstrom, Logan and Tran, Brandon and Madry, Aleksander},
  booktitle={Advances in Neural Information Processing Systems},
  volume={32},
  year={2019}
}

@inproceedings{khosla2020supervised,
  title={Supervised contrastive learning},
  author={Khosla, Prannay and Teterwak, Piotr and Wang, Chen and Sarna, Aaron and Tian, Yonglong and Isola, Phillip and Maschinot, Aaron and Liu, Ce and Krishnan, Dilip},
  booktitle={Advances in Neural Information Processing Systems},
  volume={33},
  pages={18661--18673},
  year={2020}
}

@inproceedings{li2020dividemix,
  title={DivideMix: Learning with noisy labels as semi-supervised learning},
  author={Li, Junnan and Socher, Richard and Hoi, Steven CH},
  booktitle={International Conference on Learning Representations},
  year={2020}
}

@inproceedings{loshchilov2017sgdr,
  title={SGDR: Stochastic Gradient Descent with Warm Restarts},
  author={Loshchilov, Ilya and Hutter, Frank},
  booktitle={International Conference on Learning Representations},
  year={2017}
}

@article{papyan2020prevalence,
  title={Prevalence of neural collapse during the terminal phase of deep learning training},
  author={Papyan, Vardan and Han, X Y and Donoho, David L},
  journal={Proceedings of the National Academy of Sciences},
  volume={117},
  number={40},
  pages={24652--24663},
  year={2020},
  publisher={National Acad Sciences}
}

@inproceedings{szegedy2013intriguing,
  title={Intriguing properties of neural networks},
  author={Szegedy, Christian and Zaremba, Wojciech and Sutskever, Ilya and Bruna, Joan and Erhan, Dumitru and Goodfellow, Ian and Fergus, Rob},
  booktitle={International Conference on Learning Representations},
  year={2014}
}

@article{tibshirani1996regression,
  title={Regression shrinkage and selection via the lasso},
  author={Tibshirani, Robert},
  journal={Journal of the Royal Statistical Society: Series B (Methodological)},
  volume={58},
  number={1},
  pages={267--288},
  year={1996},
  publisher={Wiley Online Library}
}

@inproceedings{wang2018cosface,
  title={CosFace: Large margin cosine loss for deep face recognition},
  author={Wang, Hao and Wang, Yitong and Zhou, Zheng and Ji, Xing and Gong, Dihong and Zhou, Jing and Li, Zhifeng and Liu, Wei},
  booktitle={Proceedings of the IEEE Conference on Computer Vision and Pattern Recognition},
  pages={5265--5274},
  year={2018}
}

@inproceedings{wang2024non,
  title={Non-negative contrastive learning},
  author={Wang, Yifei and Zhang, Qi and Guo, Yisen and Wang, Yisen},
  booktitle={International Conference on Learning Representations},
  year={2024}
}

@inproceedings{wen2016discriminative,
  title={A discriminative feature learning approach for deep face recognition},
  author={Wen, Yandong and Zhang, Kaipeng and Li, Zhifeng and Qiao, Yu},
  booktitle={European Conference on Computer Vision},
  pages={499--515},
  year={2016},
  organization={Springer}
}

@inproceedings{zeiler2014visualizing,
  title={Visualizing and understanding convolutional networks},
  author={Zeiler, Matthew D and Fergus, Rob},
  booktitle={European Conference on Computer Vision},
  pages={818--833},
  year={2014},
  organization={Springer}
}

@inproceedings{zhang2017understanding,
  title={Understanding deep learning requires rethinking generalization},
  author={Zhang, Chiyuan and Bengio, Samy and Hardt, Moritz and Recht, Benjamin and Vinyals, Oriol},
  booktitle={International Conference on Learning Representations},
  year={2017}
}

@inproceedings{jing2020implicit,
  title={Implicit Rank-Minimizing Autoencoder},
  author={Jing, Li andzb, Jure and LeCun, Yann},
  booktitle={NeurIPS},
  year={2020}
}

@inproceedings{madry2018towards,
  title={Towards Deep Learning Models Resistant to Adversarial Attacks},
  author={Madry, Aleksander and Makelov, Aleksandar and Schmidt, Ludwig and Tsipras, Dimitris and Vladu, Adrian},
  booktitle={International Conference on Learning Representations},
  year={2018},
  url={https://openreview.net/forum?id=rJzIBfZAb}
}

@inproceedings{nair2010rectified,
  title={Rectified Linear Units Improve Restricted Boltzmann Machines},
  author={Nair, Vinod and Hinton, Geoffrey E.},
  booktitle={Proceedings of the 27th International Conference on Machine Learning (ICML)},
  pages={807--814},
  year={2010}
}

@inproceedings{geirhos2019imagenet,
  title={ImageNet-trained CNNs are biased towards texture; increasing shape bias improves accuracy and robustness},
  author={Geirhos, Robert and Rubisch, Patricia and Michaelis, Claudio and Bethge, Matthias and Wichmann, Felix A and Brendel, Wieland},
  booktitle={International Conference on Learning Representations},
  year={2019}
}

@inproceedings{locatello2019challenging,
  title={Challenging Common Assumptions in the Unsupervised Learning of Disentangled Representations},
  author={Locatello, Francesco and Bauer, Stefan and Lucic, Mario and Raetsch, Gunnar and Gelly, Sylvain and Sch{\"o}lkopf, Bernhard and Bachem, Olivier},
  booktitle={International Conference on Machine Learning},
  year={2019}
}

@inproceedings{lezama2018ole,
  title={OL{\'e}: Orthogonal Low-rank Embedding for Image Recognition},
  author={Lezama, Jos{\'e} and Qiu, Qiang and Sapiro, Guillermo},
  booktitle={Proceedings of the IEEE Conference on Computer Vision and Pattern Recognition},
  year={2018}
}

@article{olah2020zoom,
  title={Zoom In: An Introduction to Circuits},
  author={Olah, Chris and Cammarata, Nick and Schubert, Ludwig and Goh, Gabriel and Petrov, Michael and Carter, Shan},
  journal={Distill},
  year={2020}
}

@techreport{krizhevsky2009learning,
  title={Learning multiple layers of features from tiny images},
  author={Krizhevsky, Alex and Hinton, Geoffrey and others},
  year={2009},
  institution={University of Toronto}
}
\bibliographystyle{icml2026}

\newpage
\appendix
\onecolumn
\section{Extended Related Work}
\label{app:related_work}

While Section \ref{sec:background} outlines the primary motivations for Class-Conditional Geometric Regularization (CCAR), this appendix provides a more granular discussion of its relationship to adjacent fields, specifically Metric Learning, Neural Collapse, and Disentanglement.

\subsection{Relationship to Metric Learning and Margin Losses}
Metric learning aims to learn a distance function such that similar samples are close and dissimilar ones are far. Seminal works like Center Loss \cite{wen2016discriminative} penalize the Euclidean distance between deep features and their class centroids, minimizing intra-class variance. ArcFace \cite{deng2019arcface} and CosFace \cite{wang2018cosface} enforce angular margins on the hypersphere to maximize inter-class separability.

\textbf{Contrast with CCAR.} While these methods effectively structure the feature space, they operate on the relative metric between points (Euclidean or Angular distance) rather than the absolute coordinate representation. For example, ArcFace allows class clusters to rotate arbitrarily as long as the angular margin is preserved. CCAR is fundamentally different as it enforces an axis-aligned constraint. By forcing class features into specific subspaces (sets of basis vectors), CCAR essentially constrains the rotation of the latent space, ensuring that specific neurons correspond to specific class manifolds. This coordinate-level constraint is what enables the emergent robustness to noise, as noise in forbidden dimensions is explicitly suppressed, whereas metric learning methods permit noise as long as the angular margin remains sufficient.

\subsection{Relationship to Neural Collapse}
Recent theoretical work has identified a phenomenon known as Neural Collapse \cite{papyan2020prevalence}, which occurs during the terminal phase of training with Cross-Entropy loss. Neural Collapse is characterized by (i) intra-class variation vanishing (features collapse to the class mean), and (ii) class means forming an optimal Simplex Equiangular Tight Frame (ETF).

\textbf{Contrast with CCAR.} While Neural Collapse describes the natural tendency of unregularized networks to compress features, it typically leads to a Simplex geometry where all dimensions are coupled. CCAR actively intervenes in this process to enforce a Block-Diagonal geometry instead of a Simplex ETF. By explicitly penalizing off-diagonal energy (via Theorem \ref{thm:trace_minimization}), CCAR prevents the distribution of information across all dimensions, preserving a structured sparsity that the natural Neural Collapse process does not guarantee. This structured collapse is crucial for robustness, as it leaves margin regions of near-zero energy that absorb adversarial perturbations.

\subsection{Relationship to Disentangled Representation Learning}
Disentanglement learning, often explored via VAEs (e.g., $\beta$-VAE) \cite{higgins2017beta}, aims to separate distinct generative factors of variation (e.g., pose, color) into independent latent dimensions.

\textbf{Contrast with CCAR.} Standard disentanglement is typically unsupervised and seeks to separate latent attributes. CCAR can be viewed as a form of Supervised Disentanglement, where the factor of variation being disentangled is the class identity itself. This addresses the theoretical impossibility of unsupervised disentanglement without inductive biases \citep{locatello2019challenging}. Unlike VAEs, which rely on reconstruction objectives and KL-divergence penalties, CCAR achieves this separation discriminatively through the subspace masking mechanism. This avoids the trade-off between reconstruction quality and disentanglement often seen in generative models, allowing CCAR to maintain high classification accuracy while ensuring that class identity is perfectly disentangled from other nuisance variations in the forbidden subspaces.

\subsection{Relationship to Adversarial Training}
Adversarial Training \cite{madry2018towards} is the gold standard for robustness, formulating training as a min-max game against a worst-case perturbation.

\textbf{Contrast with CCAR.} Adversarial Training is computationally expensive (requiring iterative gradient steps to generate attacks) and often overfits to the specific perturbation model (e.g., $L_\infty$ ball). CCAR offers an efficient alternative: rather than simulating attacks, it constructs a geometry of resistance. By minimizing the dimensionality of the manifold occupied by each class, CCAR statistically reduces the probability that a random or directed perturbation will align with the target class subspace. While Adversarial Training builds a wall at the specific boundary, CCAR moves the data manifold away from the decision boundaries entirely, providing a more general form of stability at negligible computational cost.

\section{Proofs}
\label{app:proofs}

\subsection{Proof of Theorem \ref{thm:trace_minimization}}
\label{proof:theorem1}

\begin{proof}
Let $P_{\mathcal{O}_c} = \text{diag}(M(c))$ be the orthogonal projection matrix onto the forbidden subspace indices for class $c$. The CCAR-L2 objective for a feature vector $h$ is defined as $\mathcal{L}(h, c) \propto \| P_{\mathcal{O}_c} h \|_2^2$. We expand the expected loss over the class-conditional distribution:

\begin{align}
\mathbb{E}_{h|c} [\mathcal{L}(h, c)] 
&= \mathbb{E}_{h|c} \left[ (P_{\mathcal{O}_c} h)^\top (P_{\mathcal{O}_c} h) \right] \\
&= \mathbb{E}_{h|c} \left[ h^\top P_{\mathcal{O}_c}^\top P_{\mathcal{O}_c} h \right] \\
&= \mathbb{E}_{h|c} \left[ h^\top P_{\mathcal{O}_c} h \right] \quad \text{(Idempotence: } P^\top P = P \text{)} \\
&= \mathbb{E}_{h|c} \left[ \text{Tr}( h^\top P_{\mathcal{O}_c} h ) \right] \quad \text{(Scalar trace identity)} \\
&= \mathbb{E}_{h|c} \left[ \text{Tr}( P_{\mathcal{O}_c} h h^\top ) \right] \quad \text{(Cyclic property)} \\
&= \text{Tr}\left( P_{\mathcal{O}_c} \mathbb{E}_{h|c}[ h h^\top ] \right) \quad \text{(Linearity of Expectation)}
\end{align}

We invoke the definition of the second moment matrix $\mathbb{E}[XX^\top] = \text{Cov}(X) + \mathbb{E}[X]\mathbb{E}[X]^\top$. Substituting the class covariance $\Sigma_c$ and class mean $\mu_c$:

\begin{align}
\text{Tr}\left( P_{\mathcal{O}_c} (\Sigma_c + \mu_c \mu_c^\top) \right) 
&= \text{Tr}( P_{\mathcal{O}_c} \Sigma_c ) + \text{Tr}( P_{\mathcal{O}_c} \mu_c \mu_c^\top ) \\
&= \text{Tr}( P_{\mathcal{O}_c} \Sigma_c ) + \text{Tr}( \mu_c^\top P_{\mathcal{O}_c} \mu_c ) \\
&= \text{Tr}( P_{\mathcal{O}_c} \Sigma_c ) + \| P_{\mathcal{O}_c} \mu_c \|_2^2
\end{align}

The objective decouples into two terms: (1) the trace of the projected covariance and (2) the projected energy of the mean.
\begin{equation}
\mathbb{E}[\mathcal{L}] \propto \text{Tr}( P_{\mathcal{O}_c} \Sigma_c ) + \| P_{\mathcal{O}_c} \mu_c \|_2^2
\end{equation}
Under the alignment assumption stated in Theorem \ref{thm:trace_minimization} (where $\|P_{\mathcal{O}_c}\mu_c\|^2 \to 0$), the second term vanishes. Thus, minimizing the loss becomes equivalent to minimizing the covariance trace:
\begin{equation}
\min \mathbb{E}_{h|c} [\mathcal{L}(h, c)] \iff \min \text{Tr}( P_{\mathcal{O}_c} \Sigma_c ).
\end{equation}
\end{proof}

\subsection{Proof of Theorem \ref{thm:stability_bound}}
\label{proof:theorem2}

\begin{proof}
Let $w_y$ and $w_k$ be the weight vectors for the correct class $y$ and target class $k$. We analyze the system at the converged state where subspace alignment holds ($w_c \in \mathcal{I}_c$). 

Justification for Alignment: Since the CCAR encoder suppresses feature activity outside $\mathcal{I}_c$, any weight component orthogonal to $\mathcal{I}_c$ contributes no discriminative signal (i.e., gradient from Cross-Entropy is zero) but is still penalized by standard $L_2$ weight decay. Consequently, optimization drives these orthogonal weight components to zero, ensuring $w_c$ aligns with $\mathcal{I}_c$.

Assuming this alignment and ideal feature placement $h \in \mathcal{I}_y$ (thus $w_k^\top h = 0$), a misclassification occurs under perturbation $\delta$ if:
\begin{equation}
w_k^\top (h + \delta) > w_y^\top (h + \delta)
\end{equation}
Expanding the inner products:
\begin{align}
w_k^\top h + w_k^\top \delta &> w_y^\top h + w_y^\top \delta \\
0 + w_k^\top \delta &> w_y^\top h + w_y^\top \delta \quad \text{(Orthogonality)}
\end{align}
We decompose $\delta$ relative to the subspace partition $\mathcal{P}$:
\begin{equation}
\delta = \delta_{\mathcal{I}_k} + \delta_{\mathcal{I}_y} + \sum_{j \neq k, y} \delta_{\mathcal{I}_j}
\end{equation}
Substituting this decomposition into the weight-perturbation dot products:
\begin{align}
w_k^\top \delta &= w_k^\top \delta_{\mathcal{I}_k} + w_k^\top \delta_{\mathcal{I}_y} + \dots = w_k^\top \delta_{\mathcal{I}_k} \quad \text{(since } w_k \perp \mathcal{I}_{j \neq k} \text{)} \\
w_y^\top \delta &= w_y^\top \delta_{\mathcal{I}_y} + w_y^\top \delta_{\mathcal{I}_k} + \dots = w_y^\top \delta_{\mathcal{I}_y}
\end{align}
The inequality simplifies to:
\begin{equation}
w_k^\top \delta_{\mathcal{I}_k} > w_y^\top h + w_y^\top \delta_{\mathcal{I}_y}
\end{equation}
Applying the Cauchy-Schwarz inequality ($|u^\top v| \le \|u\| \|v\|$) to the LHS term $w_k^\top \delta_{\mathcal{I}_k}$:
\begin{equation}
\|w_k\| \|\delta_{\mathcal{I}_k}\| \ge w_k^\top \delta_{\mathcal{I}_k} > w_y^\top (h + \delta_{\mathcal{I}_y})
\end{equation}
Dividing by $\|w_k\|$ yields the lower bound on the forbidden energy:
\begin{equation}
\|\delta_{\mathcal{I}_k}\| > \frac{w_y^\top h + w_y^\top \delta_{\mathcal{I}_y}}{\|w_k\|}
\end{equation}
This confirms that for misclassification to occur, the projection of noise onto the forbidden subspace $\|\delta_{\mathcal{I}_k}\|$ must exceed the geometric margin defined by the RHS.
\end{proof}

\subsection{Proof of Lemma~\ref{lem:noise_filtering}}
\label{proof:lemma1}
\begin{proof}
Let the perturbation vector $\delta \in \mathbb{R}^D$ be sampled from an isotropic Gaussian distribution $\delta \sim \mathcal{N}(0, \frac{\sigma^2}{D}I_D)$. We analyze the energy projected onto a specific class-conditional subspace $\mathcal{I}_k$ with cardinality $|\mathcal{I}_k| = K = D/C$. The projection $\delta_{\mathcal{I}_k}$ consists of $K$ independent components $\delta_j \sim \mathcal{N}(0, \frac{\sigma^2}{D})$. Standardizing these variables as $z_j = \frac{\sqrt{D}}{\sigma}\delta_j \sim \mathcal{N}(0, 1)$, we express the projected energy as:
\begin{equation}
\|\delta_{\mathcal{I}_k}\|^2 = \sum_{j \in \mathcal{I}_k} \delta_j^2 = \frac{\sigma^2}{D} \sum_{j=1}^K z_j^2 = \frac{\sigma^2}{D} Q_K,
\end{equation}
where $Q_K$ follows a Chi-squared distribution with $K$ degrees of freedom. We seek to bound the probability that this energy exceeds the geometric margin $\tau$:
\begin{equation}
P(\|\delta_{\mathcal{I}_k}\|^2 \ge \tau) = P\left( Q_K \ge \frac{D\tau}{\sigma^2} \right).
\end{equation}
Applying the Chernoff bound with parameter $t > 0$:
\begin{align}
P\left( Q_K \ge \frac{D\tau}{\sigma^2} \right) &\le \inf_{0 < t < 1/2} \frac{\mathbb{E}[e^{tQ_K}]}{e^{t \frac{D\tau}{\sigma^2}}} \\
&= \inf_{0 < t < 1/2} \exp\left( -t \frac{D\tau}{\sigma^2} - \frac{K}{2}\ln(1-2t) \right).
\end{align}
Minimizing the exponent $g(t)$ with respect to $t$ yields the optimal parameter $t^*$:
\begin{equation}
\frac{\partial g}{\partial t} = -\frac{D\tau}{\sigma^2} + \frac{K}{1-2t} = 0 \implies t^* = \frac{1}{2}\left(1 - \frac{K\sigma^2}{D\tau}\right).
\end{equation}
Substituting $K = D/C$, we have $t^* = \frac{1}{2}(1 - \frac{\sigma^2}{C\tau})$. The condition $\tau > \frac{\sigma^2}{C}$ ensures $0 < t^* < 1/2$. Substituting $t^*$ back into the bound:
\begin{align}
P(\|\delta_{\mathcal{I}_k}\|^2 \ge \tau) &\le \exp\left( -\frac{D\tau}{2\sigma^2}\left(1 - \frac{\sigma^2}{C\tau}\right) - \frac{D}{2C}\ln\left(\frac{\sigma^2}{C\tau}\right) \right) \\
&= \exp\left( -D \left[ \frac{\tau}{2\sigma^2} - \frac{1}{2C} + \frac{1}{2C}\ln\left(\frac{C\tau}{\sigma^2}\right)^{-1} \right] \right) \\
&= \exp\left( -D \left[ \frac{1}{2C}\left(\frac{C\tau}{\sigma^2} - 1 - \ln\frac{C\tau}{\sigma^2}\right) \right] \right).
\end{align}
Defining the rate constant $\mathcal{A} = \frac{1}{2C}(\frac{C\tau}{\sigma^2} - 1 - \ln\frac{C\tau}{\sigma^2})$, which is strictly positive for $\frac{C\tau}{\sigma^2} > 1$, we conclude:
\begin{equation}
P(\|\delta_{\mathcal{I}_k}\|^2 \ge \tau) \le \exp(-D \cdot \mathcal{A}).
\end{equation}
\end{proof}

\section{Regularization Design Analysis}
\label{app:regularization}

In Section \ref{l2choice}, we motivated the selection of the squared $L_2$ norm for the geometric penalty based on its optimization stability and smooth gradient field. To justify this design, we provide a comparative analysis of several candidate regularizers representing distinct geometric inductive biases. We evaluate these variants based on their ability to maintain standard accuracy while providing stability under severe supervision corruption.

\subsection{Experimental Configuration}
For the supervision corruption tasks, we first train the ResNet-18 \cite{he2016deep} backbone from scratch for 200 epochs on CIFAR-100 and ImageNet-100 \cite{deng2009imagenet}. For the experiments on CIFAR-100, we employ a batch size of 256, a weight decay of 0.0005, and an initial learning rate of 0.1 regulated by a Cosine Annealing schedule. For ImageNet-100, we utilize a batch size of 256 and a weight decay of 0.0001 while maintaining the same learning rate and scheduler settings. We apply the regularization penalty $\lambda=3$ from the initial epoch. The value of the regularization coefficient $\lambda$ was determined through an extensive grid search over the range of $[0.1, 5]$, where $\lambda=3$ was found to provide the superior balance between preserving standard classification performance and maximizing robustness against high-intensity label noise. To ensure fair comparison, we configured all baseline methods using the optimal hyperparameter settings reported in their respective original publications (e.g., center weight $\lambda=0.003$ for Center Loss, temperature $\tau=0.07$ for SimCLR/NCL), ensuring they operate at their established state-of-the-art capacity. Linear evaluation is performed by freezing the backbone and training a linear classifier for 50 epochs across a noise spectrum ranging from 0\% to 90\% in 10\% increments.

\subsection{Dynamics of Geometric Penalty Variants}

The \textbf{$L_1$ Sparsity Baseline} serves as a standard choice for inducing sparsity within the representation \cite{tibshirani1996regression}. While it effectively isolates feature activations, it possesses a non-differentiable gradient at the origin, which can potentially destabilize the optimization process during the refinement of fine-grained features.
\begin{equation}
\mathcal{L}_{L_1} = \|\mathbf{M}(y) \odot \mathbf{h}\|_1
\end{equation}

The \textbf{Cosine Margin Penalty} borrows principles from margin-based metric learning. This formulation penalizes representations based on the cosine similarity between the feature vector and competing class centroids relative to the target centroid, incorporating a margin $\delta$ to enforce separation.
\begin{equation}
\mathcal{L}_{\text{Margin}} = \max_{k \neq y}(\cos(h, \mu_k)) - \cos(h, \mu_y) + \delta
\end{equation}

The \textbf{Subspace Energy Ratio} constrains the logarithmic ratio of energy residing in the forbidden subspace relative to the energy within the active subspace. This approach focuses on the relative distribution of energy across the feature manifold rather than penalizing absolute activation magnitudes.
\begin{equation}
\mathcal{L}_{\text{Ratio}} = \log\left(1 + \frac{\|h_{\text{off}}\|_2^2}{\|h_{\text{aligned}}\|_2^2}\right)
\end{equation}

The \textbf{Orthogonal Projection Penalty} represents a direct subspace constraint. It penalizes any component of the feature vector $h$ that resides in the orthogonal complement of the class-specific subspace, as defined by a projection matrix $P_y$ derived from class-specific statistics.
\begin{equation}
\mathcal{L}_{\text{Proj}} = \|(I - P_y)h\|_2^2
\end{equation}

The \textbf{Squared $L_2$ Penalty (Selected Design)} utilizes a smooth quadratic penalty to suppress activations in forbidden dimensions. By indexing these dimensions via a class-conditional mask $M(y)$, the objective provides a restorative force that aligns the latent topology with the semantic hierarchy of the dataset.
\begin{equation}
\mathcal{L}_{L_2} = \frac{1}{|\mathcal{O}_y|} \sum_{j \in \mathcal{O}_y} h_j^2
\end{equation}

\textbf{Results.} As shown in Table \ref{tab:ablation_variants_transposed}, the $L_2$ penalty consistently outperforms all alternative geometric constraints across the noise spectrum. For example, at 90\% noise, the $L_2$ formulation achieves 74.21\% accuracy, significantly surpassing the $L_1$ baseline (69.14\%) and the runner-up Subspace Energy Ratio (73.29\%). This indicates that while sparsity-inducing ($L_1$) or relative ($L_{\text{Ratio}}$) penalties offer some resilience, the smooth gradient field of the quadratic penalty provides superior optimization stability under extreme corruption. Therefore, we conclude that the  $L_2$ norm is the optimal design choice for implementing the CCAR objective, as it maximizes robustness without compromising the representational capacity required for clean classification.

\begin{table*}[t]
\caption{Comparison of Linear Probing Accuracy (\%) across geometric penalty variants on CIFAR-100. By transposing the view, we observe the stability of each method as noise increases. The $L_2$ Penalty consistently outperforms sparsity ($L_1$) and metric-based constraints. Results are averaged over multiple seeds (mean $\pm$ std).}
\label{tab:ablation_variants_transposed}
\begin{center}
\begin{small}
\begin{tabular}{cccccc}
\toprule
\textsc{Noise (\%)} & \textsc{$L_1$ Baseline} & \textsc{Cosine Margin} & \textsc{Energy Ratio} & \textsc{Orthogonal Proj.} & \textbf{\textsc{$L_2$ Penalty}} \\
\midrule
0  & 73.08$\pm$0.21 & 74.73$\pm$0.39 & 74.13$\pm$0.32 & 75.65$\pm$0.28 & \textbf{75.68$\pm$0.15} \\
10 & 72.99$\pm$0.15 & 74.62$\pm$0.12 & 74.10$\pm$0.36 & 75.02$\pm$0.28 & \textbf{75.41$\pm$0.31} \\
20 & 72.88$\pm$0.11 & 74.46$\pm$0.39 & 73.96$\pm$0.35 & 74.99$\pm$0.16 & \textbf{75.26$\pm$0.15} \\
30 & 72.80$\pm$0.16 & 74.51$\pm$0.19 & 73.92$\pm$0.26 & 75.07$\pm$0.23 & \textbf{75.26$\pm$0.19} \\
40 & 72.93$\pm$0.28 & 74.46$\pm$0.14 & 74.03$\pm$0.19 & 75.08$\pm$0.21 & \textbf{75.25$\pm$0.24} \\
50 & 72.78$\pm$0.34 & 74.28$\pm$0.16 & 73.85$\pm$0.25 & 74.97$\pm$0.28 & \textbf{75.15$\pm$0.11} \\
60 & 72.73$\pm$0.28 & 74.25$\pm$0.15 & 73.74$\pm$0.12 & 74.73$\pm$0.38 & \textbf{75.11$\pm$0.39} \\
70 & 72.62$\pm$0.34 & 74.06$\pm$0.19 & 73.81$\pm$0.13 & 74.69$\pm$0.31 & \textbf{75.00$\pm$0.23} \\
80 & 72.36$\pm$0.14 & 73.81$\pm$0.25 & 73.73$\pm$0.11 & 74.02$\pm$0.37 & \textbf{74.67$\pm$0.18} \\
90 & 69.14$\pm$0.30 & 72.75$\pm$0.19 & 73.29$\pm$0.26 & 72.89$\pm$0.26 & \textbf{74.21$\pm$0.16} \\
\bottomrule
\end{tabular}
\end{small}
\end{center}
\vskip -0.1in
\end{table*}

\section{Additional Empirical Analysis}
\label{app:additional_analysis}

In Section \ref{section4}, we established that minimizing the CCAR objective theoretically compresses the covariance trace in forbidden directions (Theorem \ref{thm:trace_minimization}) and partitions the latent space into orthogonal subspaces (Theorem \ref{thm:stability_bound}). In this section, we provide empirical verification of these geometric phenomena using diagnostic metrics computed on the ImageNet-100 validation set.

\subsection{Fisher Discriminant Analysis}

To quantitatively monitor the emergence of the block-diagonal structure, we track the Fisher Discriminant Ratio (FDR) throughout the training trajectory. The FDR serves as a proxy for linear separability, defined as the ratio of the inter-class scatter trace to the intra-class scatter trace \cite{fisher1936use}: $J = \text{Tr}(\Sigma_b) / \text{Tr}(\Sigma_w)$.

\begin{figure}[ht]
\vskip 0.1in
\begin{center}
    \includegraphics[width=0.4\columnwidth]{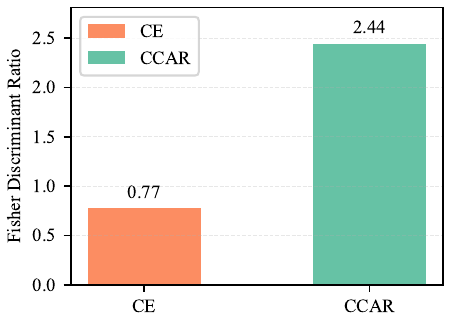}
    \caption{Fisher Ratio analysis on CIFAR-100. CCAR ($J=2.44$) shows significantly stronger class separation than CE ($J=0.77$).}
    \label{fig:fisher_ratio}
\end{center}
\vskip -0.1in 
\end{figure}

\textbf{Results.} As illustrated in Figure \ref{fig:fisher_ratio}, CCAR maintains a significantly higher Fisher Ratio compared to the Cross-Entropy (CE) baseline throughout the training process. Notably, while the FDR for the baseline degrades rapidly as supervision noise increases, CCAR exhibits robust separability. This confirms that the explicit energy minimization penalty continuously suppresses intra-class variance ($\Sigma_w$), effectively maximizing the FDR even when the label signal is unreliable.

\subsection{Verification of Covariance Compression}
\label{app:covariance_compression}

Theorem \ref{thm:trace_minimization} states that the objective minimizes the trace of the intra-class covariance matrix in forbidden directions. We verify the effectiveness of this minimization by analyzing the Mean Cluster Radii for the 100 classes in ImageNet-100.

\begin{figure}[ht]
\vskip 0.1in
\begin{center}
    \includegraphics[width=0.6\columnwidth]{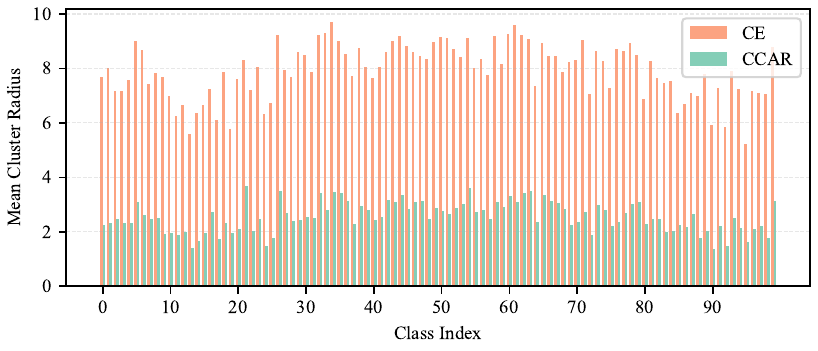}
    \caption{Impact of CCAR on Intra-Class Compactness: Mean Cluster Radii Analysis on ImageNet-100. CCAR yields consistently lower radii compared to the baseline, confirming that the regularization effectively compresses intra-class variance and induces tighter class manifolds.}
    \label{fig:cluster_radii}
\end{center}
\vskip -0.1in
\end{figure}

\textbf{Results.} As illustrated in Figure \ref{fig:cluster_radii}, the clusters formed by CCAR are generally more compact than those of the baseline. The reduction in mean radius indicates that the regularization successfully compresses intra-class variance. This compression appears to selectively target the directions of variation that do not align with the class-specific subspaces, resulting in tighter class manifolds that are spatially distant from the decision boundaries of competing classes.

\section{Empirical Verifications of Theoretical Analysis on CCAR}
\label{app:empirical_verification}

In Section \ref{section4}, we established that minimizing the CCAR objective theoretically compresses the covariance trace in forbidden directions (Theorem \ref{thm:trace_minimization}) and partitions the latent space into optimal orthogonal subspaces. In this section, we provide empirical verification of these phenomena using diagnostic metrics on ImageNet-100.

\subsection{Verification of Optimal Representations}
\label{app:optimal_representations}

To verify the theoretical analysis of the optimal solutions derived in Theorem \ref{thm:trace_minimization} (Trace Minimization), we investigate the internal topology of the learned representations. We examine two complementary metrics: the Expected Activation (EA) distribution, which measures where feature energy is concentrated, and the Eigenvalue Distribution, which measures the effective rank and capacity of the representation.

\begin{figure}[ht]
\vskip 0.1in
\begin{center}
    \begin{minipage}[b]{0.40\columnwidth}
        \centering
        \includegraphics[width=\linewidth]{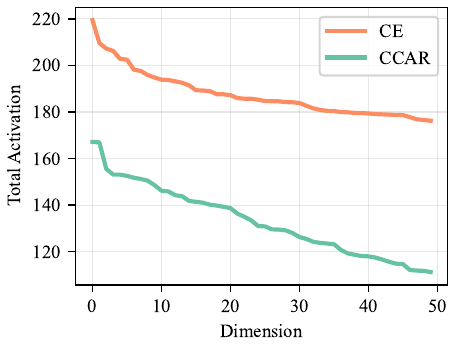}
        \centerline{\small (a) Distribution of the 50 largest EA values}
    \end{minipage}
    \hspace{0.5cm} 
    \begin{minipage}[b]{0.40\columnwidth}
        \centering
        \includegraphics[width=\linewidth]{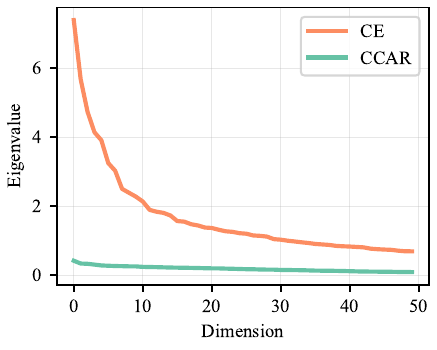}
        \centerline{\small (b) Distribution of the 50 largest eigenvalues}
    \end{minipage}
    
    \caption{Analysis of learned features of CE and CCAR on ImageNet-100.}
    \label{fig:feature_analysis}
\end{center}
\vskip -0.25in
\end{figure}

\textbf{Energy Concentration (EA).} As illustrated in Figure \ref{fig:feature_analysis}(a), we observe a distinct difference in energy allocation. The features learned by CCAR exhibit a concentrated activation profile where the majority of signal energy is contained within the top-ranked dimensions. This alignment is consistent with the subspace constraints imposed by our objective, confirming that the network effectively utilizes the designated active dimensions while suppressing variance in the forbidden directions. In contrast, the baseline CE model shows a uniform distribution of energy across the spectrum, indicating a dispersed representation without explicit subspace localization.

\textbf{Trace Minimization (Eigenvalues).} Furthermore, the eigenvalue distribution in Figure \ref{fig:feature_analysis}(b) reveals the spectral signature of trace minimization. The CCAR eigenspectrum exhibits a sharp decay followed by a flat tail near zero, confirming that variance in forbidden directions has been effectively eliminated. Crucially, the spectrum remains flat in the leading dimensions, unlike the baseline which exhibits a slower decay. This suggests that CCAR avoids dimensional collapse \citep{jing2020implicit}; instead of collapsing to a single point, the representation utilizes the full capacity of the allowed block ($I_c$) to encode intra-class variations, approximating the ideal block-diagonal geometry described in Theorem \ref{thm:trace_minimization} more closely than the unregularized baseline \citep{papyan2020prevalence}.

\section{Extended Feature Properties}
\label{app:feature_properties}

Beyond the theoretical verification, we investigate the practical properties of the learned features that emerge from this geometric structuring, specifically focusing on intrinsic linearity and semantic interpretability.

\subsection{Intrinsic Linearity and Separability}
\label{sec:intrinsic_linearity}

A key theoretical requirement for disentangled representations is intrinsic linear separability: semantic information should be accessible via a single affine transformation, without requiring complex non-linear decoding \citep{alain2016understanding}. To verify this, we compare the classification accuracy of a Linear Probe against a Multi-Layer Perceptron (MLP) Probe (two layers with ReLU \citep{nair2010rectified}) trained on frozen features.

\begin{table}[!h]
\caption{Classification accuracy (\%) of Linear vs. MLP probes. The bottom row ($\Delta$) highlights the significant stability gain of CCAR, particularly in the high-noise regime.}
\label{tab:linearity_analysis}
\begin{center}
\begin{small}
\begin{tabular}{lcccc}
\toprule
& \multicolumn{2}{c}{\textbf{\textsc{Clean Data}}} & \multicolumn{2}{c}{\textbf{\textsc{90\% Label Noise}}} \\
\cmidrule(lr){2-3} \cmidrule(lr){4-5}
\textsc{Method} & \textsc{Linear} & \textsc{MLP} & \textsc{Linear} & \textsc{MLP} \\
\midrule
Control (CE) & 77.38 & 77.34 & 61.28 & 66.36 \\
\textbf{CCAR (Ours)} & \textbf{78.16} & \textbf{78.20} & \textbf{73.18} & \textbf{76.20} \\
\midrule
$\Delta$ & \textit{+0.78} & \textit{+0.86} & \textit{+11.90} & \textit{+9.84} \\
\bottomrule
\end{tabular}
\end{small}
\end{center}
\vskip -0.1in
\end{table}

As summarized in Table \ref{tab:linearity_analysis}, CCAR consistently outperforms the Cross-Entropy (CE) baseline. On clean data, CCAR achieves a Linear accuracy of 78.16\%, surpassing the baseline's 77.38\%. This advantage becomes critical in high-noise regimes (90\% label noise), where the baseline's linear separability collapses to 61.28\%. Under these conditions, the baseline relies heavily on the MLP to recover performance (+5.08\% boost), indicating that its feature manifold becomes entangled and curved. In contrast, CCAR maintains a high linear accuracy of 73.18\%, with the linear probe capturing the vast majority of the semantic structure. This confirms that the CCAR manifold remains intrinsically flat and robust, even when the supervision signal is heavily corrupted.

\subsection{Visual Interpretability}
\label{app:visual_interpretability}

Finally, we provide a qualitative assessment of the semantic information encoded within the learned features. We visualize the top-activating images for specific neurons to determine if the geometric constraints translate into semantic consistency \cite{zeiler2014visualizing}. We select neurons based on their maximum activation values on the test set and display the images that elicit the strongest response, a method akin to Network Dissection \cite{bau2017network}.

The visualizations in Figure \ref{fig:vis_cifar} and Figure \ref{fig:vis_imagenet} reveal a high degree of semantic coherence in the CCAR representations. We observe that neurons within the assigned subspaces tend to fire exclusively for concepts related to that class. For example, specific neurons within a vehicle-assigned subspace consistently activate on wheel patterns or metallic textures, while remaining inactive for biologically textured images. This aligns with findings that standard CNNs are biased towards texture rather than shape \citep{geirhos2019imagenet}.

\begin{figure*}[h]
\vskip 0.1in
\begin{center}
    \begin{minipage}[b]{0.48\textwidth}
        \centering
        \includegraphics[width=\linewidth]{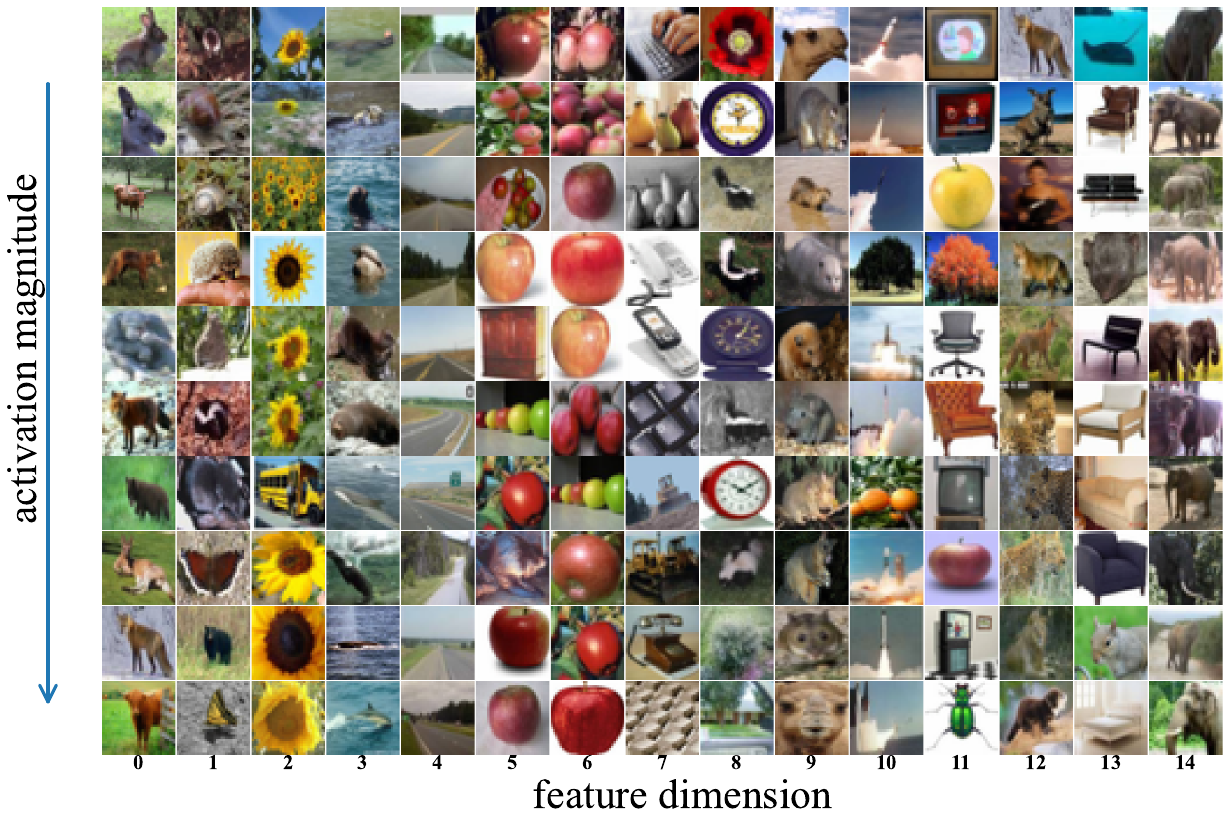}
        \centerline{\small (a) Cross Entropy}
    \end{minipage}
    \hfill
    \begin{minipage}[b]{0.48\textwidth}
        \centering
        \includegraphics[width=\linewidth]{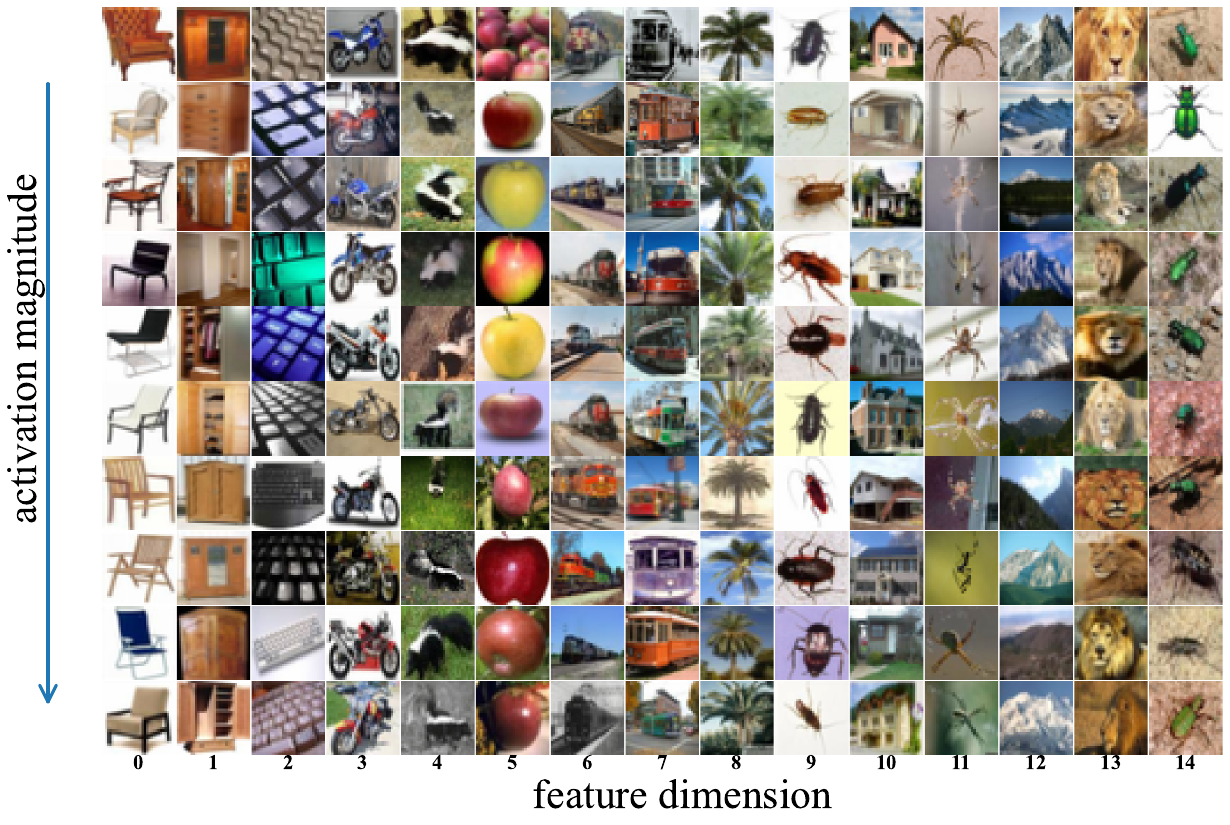}
        \centerline{\small (b) Class-Conditional Activation Regularization (Ours)}
    \end{minipage}
    
    \caption{Visualization of CIFAR-100 test samples with the largest values along each feature dimension (sorted according to activation values).}
    \label{fig:vis_cifar}
\end{center}
\vskip -0.1in
\end{figure*}

Crucially, this visualization highlights the property of intra-class polysemanticity discussed in Section \ref{section3}. While the subspace is class-specific, the individual neurons within that subspace capture diverse attributes of the class. One neuron may specialize in the frontal view of an object, while another captures the side profile. This contrasts with the baseline, where

\begin{figure*}[h]
\vskip 0.1in
\begin{center}
    \begin{minipage}[b]{0.48\textwidth}
        \centering
        \includegraphics[width=\linewidth]{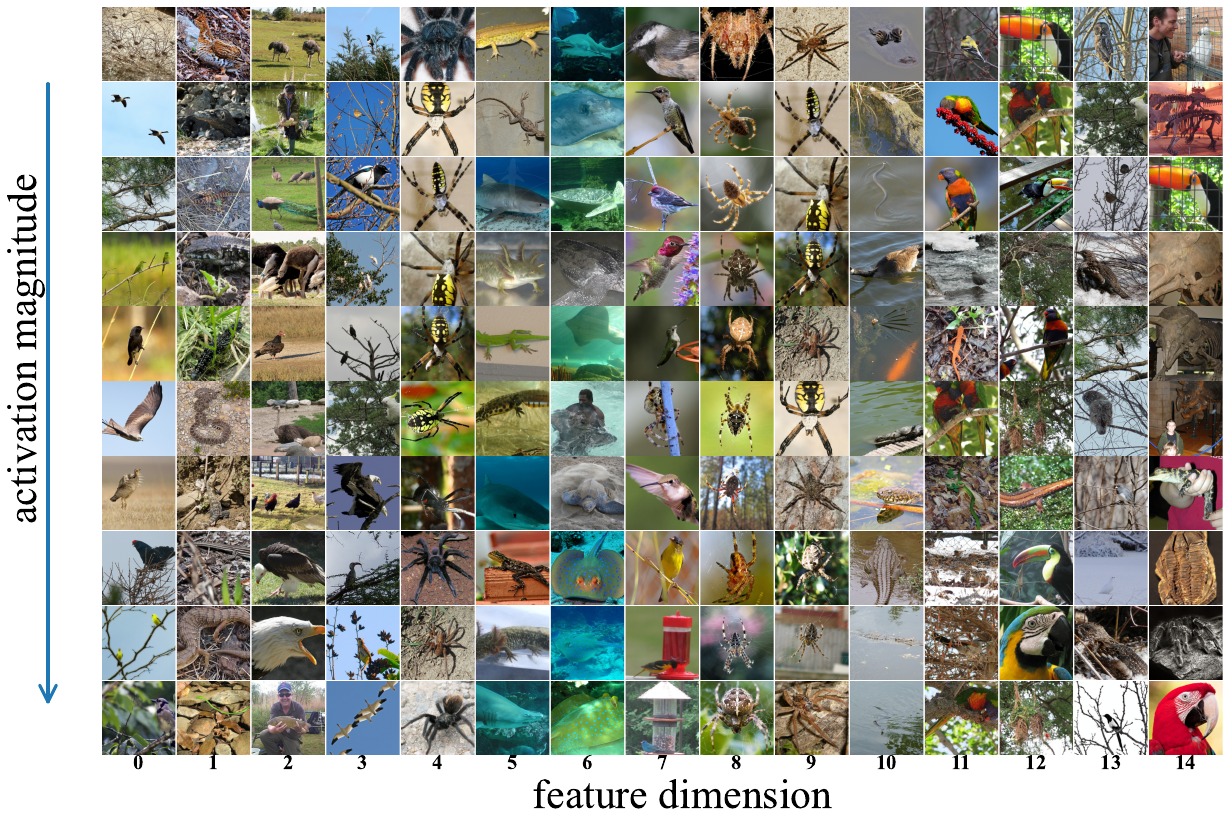} 
        \centerline{\small (a) Cross Entropy}
    \end{minipage}
    \hfill
    \begin{minipage}[b]{0.48\textwidth}
        \centering
        \includegraphics[width=\linewidth]{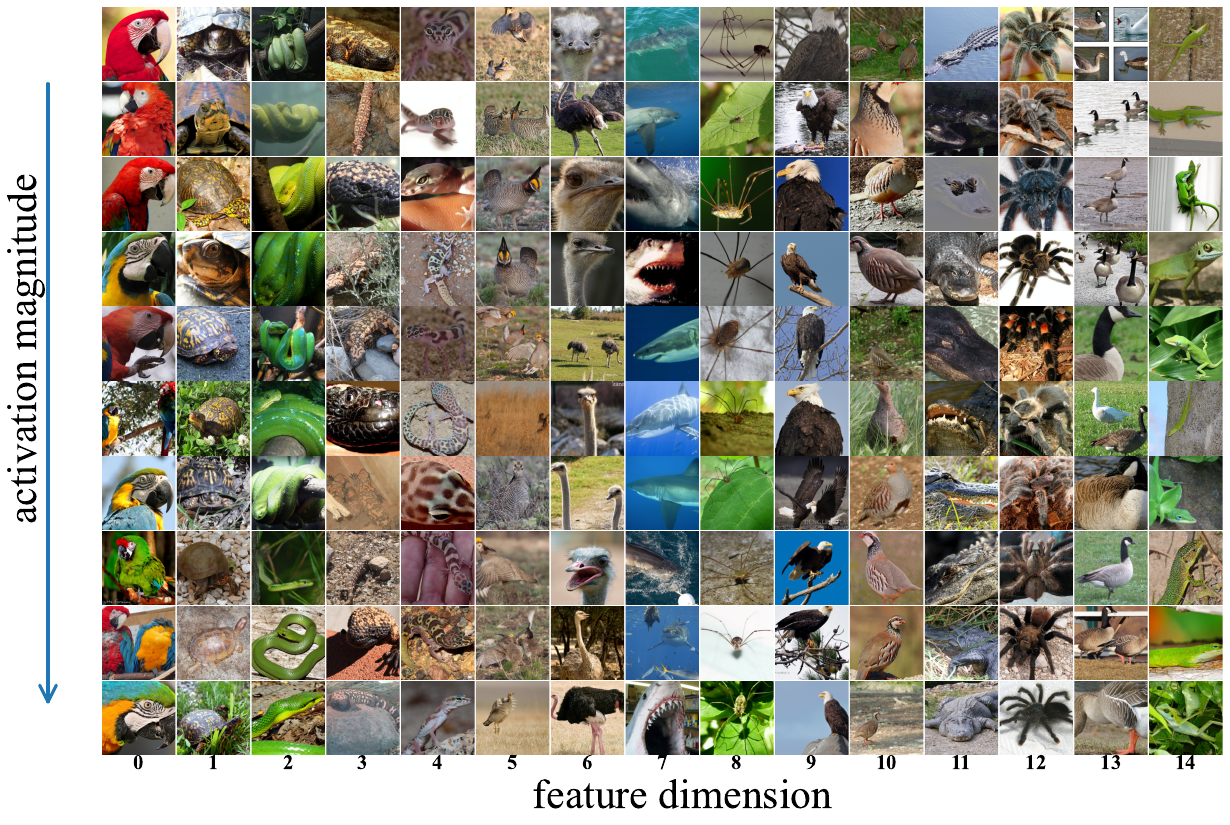}
        \centerline{\small (b) Class-Conditional Activation Regularization (Ours)}
    \end{minipage}
    
    \caption{Visualization of ImageNet-100 test samples with the largest values along each feature dimension (sorted according to activation values).}
    \label{fig:vis_imagenet}
\end{center}
\vskip -0.1in
\end{figure*}

top activating images for a single neuron often span multiple semantically unrelated classes. The semantic purity observed in CCAR confirms that the geometric regularization forces the network to align its internal feature axes with meaningful semantic factors, effectively filtering out spurious correlations that often bridge different classes in unregularized models.

\section{Details of Experiment Setup}
\label{app:experiment_setup}

\subsection{Implementation and Evaluation Protocol}
\label{g1}
In our implementation, the encoder $f_\theta$ consists of a ResNet-18 backbone $g$ \cite{he2016deep}. For all experiments, we apply the Class-Conditional Geometric Regularization at the penultimate feature layer $h = g(x)$. The regularization term $\mathcal{L}_{\text{CCAR}}$ is computed by applying class-conditional binary masks $M(y)$ to these activations. We select the regularization strength $\lambda=3$ based on a pilot study on CIFAR-100 where it demonstrated the optimal balance between subspace compression and gradient stability.

\textbf{Evaluation.} When evaluating representation quality, we discard the final classification weights and utilize the backbone outputs. For linear probing, we train a single linear layer for 50 epochs with a learning rate of 0.1 and Cosine Annealing \cite{loshchilov2017sgdr}. For the manifold linearity experiments in Appendix \ref{app:feature_properties}, we utilize a two-layer MLP probe with a hidden dimension of 2048 to assess if non-linear decoding can recover information from entangled baseline features.

\subsection{Geometric Analysis Implementation}
\label{g2}
In this section, we define the geometric metrics utilized for the mechanism verification in Section 5.2 and the theoretical proofs in Appendix \ref{app:empirical_verification}. All metrics are computed on the ImageNet-100 validation set using the ResNet-18 backbone trained for 200 epochs.

To visualize the internal topology of the latent manifold, we compute the \textbf{Feature Correlation Matrix} ($C$). Each entry $C_{ij}$ represents the cosine similarity between the activations of feature dimensions $i$ and $j$, computed over the expectation of the validation set:
\begin{equation}
C_{ij} = \frac{|\mathbb{E}[h_i^\top h_j]|}{\|h_i\|\|h_j\|}
\end{equation}
To reveal the block-diagonal structure, we sort the feature dimensions according to their assigned class indices. Dimensions with zero activation variance (dead neurons) are excluded from this visualization to ensure the heatmap accurately reflects the geometry of the active manifold.

We quantify the representational efficiency of the encoder using \textbf{Feature Sparsity}, which measures the fraction of the feature space that remains inactive for a given input. A dimension $h_j$ is considered inactive if its magnitude falls below a numerical threshold $\epsilon = 1e^{-5}$. The sparsity score $S(x)$ for a sample $x$ is defined as:
\begin{equation}
S(x) = \frac{1}{D} \sum_{j=1}^D \mathbb{I}(|h_j(x)| < \epsilon)
\end{equation}
We report the mean sparsity averaged over the entire test set.

To assess the semantic purity of the learned features, we utilize the \textbf{Class Consistency Rate (CCR)}. This metric evaluates whether individual neurons specialize in concepts relevant to their assigned class subspace. For each neuron $j$, we retrieve the set of top-$K$ maximally activating images $\{x^{(1)}, \dots, x^{(K)}\}$ from the validation set (with $K=10$). The consistency score is calculated as the proportion of these images that belong to the class $c$ assigned to that neuron's subspace $\mathcal{I}_c$:
\begin{equation}
\text{CCR}_j = \frac{1}{K} \sum_{k=1}^K \mathbb{I}(y(x^{(k)}) = c)
\end{equation}
The global CCR is reported as the mean score across all active neurons.

Finally, to quantitatively monitor the compression of class manifolds referenced in Appendix \ref{app:additional_analysis}, we track the \textbf{Fisher Discriminant Ratio (FDR)}. This metric serves as a proxy for linear separability and is defined as the ratio of the trace of the inter-class scatter matrix ($\Sigma_b$) to the intra-class scatter matrix ($\Sigma_w$) \cite{fisher1936use}:
\begin{equation}
\text{FDR} = \frac{\text{Tr}(\Sigma_b)}{\text{Tr}(\Sigma_w)}
\end{equation}
We track this metric at regular intervals during training; high values indicate that the class manifolds have collapsed into compact, widely separated clusters, while low values indicate diffuse and overlapping distributions.

\subsection{Experiment Details of Input Stability}
\label{g3}
To evaluate the perturbation stability properties derived in Theorem \ref{thm:stability_bound}, we subject models trained under 40\% symmetric label noise to diverse input-level corruptions. We characterize robustness against random pixel-level distortions by adding isotropic Gaussian noise $\delta \sim \mathcal{N}(0, \sigma^2 I)$ to the normalized test samples. Adversarial stability is assessed through both single-step and iterative optimization-based attacks. We utilize the Fast Gradient Sign Method (FGSM) \cite{goodfellow2014explaining} with perturbation magnitudes $\epsilon \in \{2/255, 4/255, 8/255\}$.

Iterative robustness is measured using Projected Gradient Descent (PGD) \cite{madry2018towards} with 20 iterations, employing a total budget of $\epsilon=8/255$ and a step size of $\alpha=2/255$. All robustness evaluations are performed strictly on the ResNet-18 backbones preserved from the training phase to ensure that the reported metrics reflect the intrinsic stability of the learned geometric manifold rather than the capacity of a specific classification head. It is worth noting that we maintain the same data normalization and resolution parameters used during the training phase to ensure that the results are not biased by variations in the input pipeline.

\subsection{Experiment Details of Supervision Corruption}
\label{g4}
For the supervision corruption tasks, during the training process, we utilize ResNet-18 \cite{he2016deep} as the backbone and train the models from scratch for 200 epochs on CIFAR-100 \cite{krizhevsky2009learning}. and ImageNet-100 \cite{deng2009imagenet}. For the experiments on CIFAR-100, we employ a batch size of 256, a weight decay of 0.0005, and an initial learning rate of 0.1 regulated by a Cosine Annealing schedule \cite{loshchilov2017sgdr}. For ImageNet-100, we utilize a batch size of 256 and a weight decay of 0.0001 while maintaining the same learning rate and scheduler settings. We apply the regularization penalty $\lambda=3$ from the initial epoch without any warmup period.

To evaluate robustness, symmetric label noise is injected from 0\% to 90\% in 10\% increments. During the linear evaluation stage, we freeze the backbone and train a classifier for 50 epochs following the standard protocol. For the retrieval experiments, we utilize cosine similarity to identify the top-10 nearest neighbors in the feature space and report the mean Average Precision at 10. It is worth noting that we apply these settings consistently across both datasets to ensure a fair comparison of the geometric properties learned by the model.

\subsection{Baseline Implementation Details}
\label{app:baseline_implementation}
To ensure reproducibility and fair comparison within the supervised benchmark, we explicitly define the implementation of the contrastive baselines referenced in Section 5.1. Since SimCLR \cite{chen2020simple} and NCL \cite{wang2024non} are typically utilized in self-supervised or distinct learning paradigms, we adapt them to the supervised setting as \textbf{auxiliary regularization objectives}. This ensures a fair comparison with CCAR, where the geometric constraints are applied alongside the standard supervision. The total loss is formulated as $\mathcal{L}_{\text{total}} = \mathcal{L}_{\text{CE}} + \lambda \mathcal{L}_{\text{base}}$, where both objectives are optimized jointly from scratch.

\begin{itemize}
    \item \textbf{SimCLR (Contrastive):} We attach a non-linear projection head (2-layer MLP) to the backbone and minimize the NT-Xent loss on the projected features. This objective serves as a regularizer to encourage instance-level discrimination alongside the class-level supervision provided by Cross-Entropy.
    \item \textbf{NCL (Non-negative):} We follow the official implementation \cite{wang2024non}, applying non-negative constraints to the projection head weights and utilizing the NCL objective function to encourage sparsity in the learned representations.
    \item \textbf{Center Loss:} We maintain a memory bank of learnable class centroids. The model minimizes the intra-class Euclidean distance between features and their corresponding class centers, effectively reducing the intra-class variance $\Sigma_w$.
\end{itemize}
Crucially, these methods are not used for pre-training; they are integrated directly into the supervised training loop. All baselines utilize the same ResNet-18 backbone, data augmentation pipeline, and training schedule (200 epochs) as the proposed CCAR method to ensure that differences in performance are attributable to the geometric regularization strategy rather than the training protocol.

\end{document}